\documentclass{article}

\PassOptionsToPackage{numbers, compress}{natbib}
\usepackage[main, final]{neurips_2026}

\usepackage[utf8]{inputenc}
\usepackage[T1]{fontenc}
\usepackage{microtype}
\usepackage{graphicx}
\usepackage{wrapfig}
\usepackage{booktabs}
\usepackage{multirow}
\usepackage{tikz}
\usepackage{threeparttable}
\usepackage{tikz-cd,mathtools}
\usepackage{mathtools}
\usepackage{subcaption}
\usetikzlibrary{arrows, matrix} 
\usepackage{diagbox}
\usepackage{colortbl}
\usepackage{tcolorbox}
\usepackage{etoolbox}
\usepackage{hyperref}
\usepackage{url}
\usepackage{booktabs}
\usepackage{amsfonts}
\usepackage{amsmath}
\usepackage{amssymb}
\usepackage{amsthm}
\usepackage{nicefrac}
\usepackage{microtype}
\usepackage{xcolor}
\usepackage{graphicx}
\usepackage{mathtools}
\usepackage{enumitem}
\usepackage{bbm}
\usepackage{amsthm}
\usepackage{algorithm}
\usepackage{algpseudocode}
\usepackage{marvosym}
\DeclareRobustCommand{\corr}{\textsuperscript{\textrm{\Letter}}}

\title{Neural Statistical Functions}

\author{
  Daniel Xu$^1$,
  Yuxin Xie$^2$,
  Minghao Guo$^2$\corr,
  Haixu Wu$^2$\corr,
  Wojciech Matusik$^2$\\
  $^1$Columbia University, $^2$MIT CSAIL \\
  {\small \texttt{xu.daniel@columbia.edu, yxie25@mit.edu, \{guomh2014,wuhaixu98\}@gmail.com}},\\ {\texttt{wojciech@csail.mit.edu}}
}

\newcommand{\R}{\mathbb{R}}
\newcommand{\E}{\mathbb{E}}
\newcommand{\1}{\mathbbm{1}}
\newcommand{\sg}{\mathrm{sg}}
\newcommand{\dd}{\mathrm{d}}

\newcommand{\cdf}{Q_\mu}
\newcommand{\icdf}{Q_\mu^{-1}}

\newtheorem{theorem}{Theorem}
\newtheorem{proposition}[theorem]{Proposition}

\newtheorem{remark}[theorem]{Remark}

\begin{document}

\maketitle

\begin{abstract}
Classical deep learning typically operates on individual cases. Despite its success, real-world usage often requires repeated inference to estimate statistical quantities for complex decision-making tasks involving uncertainty or extreme-value analysis, resulting in substantial latency. We introduce \emph{neural statistical functions}, a new family of models learned from pre-trained single-sample predictors and scattered data samples, which can directly infer statistics over continuous operating condition ranges without explicit sampling. By introducing the notion of \emph{prefix statistics}, we transform and unify diverse statistical functions (\emph{e.g.}, integrals, quantiles, and maxima) into an interval-conditional framework, in which a principled identity between the prefix statistics and the individual-case regression serves as the learning objective. Neural statistical functions achieve strong performance in estimating essential statistics of complex physical processes, including accumulated energy in dynamical systems, quantiles of aerodynamic responses, and maximum stress in crash processes, while achieving up to a 100$\times$ reduction in model evaluations.
\end{abstract}

\section{Introduction}

Deep learning has become widely used to build neural surrogates for expensive or complex forward systems~\cite{achiam2023gpt,bi2023accurate,wu2026GeoPT,jumper2021highly,wang2023scientific}. For example, in physics simulation, given a geometry design together with an operating condition, the neural surrogate
predicts the physical response, \emph{e.g.},~pressure, stress, or energy~\cite{nabian2025automotive,wu2024Transolver,wu2026GeoPT}. While neural surrogates have substantially accelerated scientific discovery and engineering~\cite{wu2024Transolver,luotransolver++,alkin2025abupt}, the standard interface remains single-condition prediction: the model is queried at one specified condition and returns the corresponding response. However, many deployment-time queries are not single-condition predictions. In design, planning, and reliability analysis, the required output is often a statistic of the response over a range of possible conditions: accumulated cost over a time window, failure probabilities over uncertain loads, or maximum stress over impact angles. 
A standard approach is Monte Carlo sampling or dense sweeping: repeatedly evaluate the simulator or its neural surrogate at many conditions, then compute the desired statistic, which can become prohibitively expensive \cite{altair_physicsai,ansys_simai} for numerical solvers and remain non-negligible for neural surrogates that output large spatial fields \cite{zhou2026transolver}.

\begin{wrapfigure}{r}{0.28\textwidth}
\vspace{-35pt}
\begin{center}
\centerline{\includegraphics[width=0.26\textwidth]{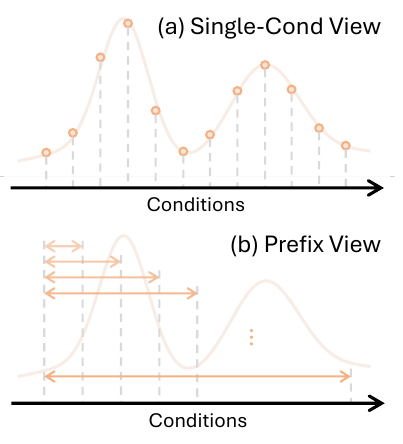}}
\vspace{-5pt}
	\caption{\small{This paper adopts a new \emph{prefix view} for statistics.}}
	\label{fig:intro}
\end{center}
\vspace{-20pt}
\end{wrapfigure}

Consider industrial crash-test design as an example. Let $\mathbf{x}\in\R^{N\times d}$ denote the spatial, geometric, or state representation at which a quantity of interest is evaluated (\emph{e.g.}, geometry of designed shapes), $c$ represents the operating conditions (\emph{e.g.},~impact angle during crash), and $\boldsymbol{h}(\mathbf{x},c)$ is the targeted quantity (\emph{e.g.},~inner stress). The standard approach to computing statistical properties $\boldsymbol{h}_{\text{Stat}}(\cdot,\cdot)$ within the varying condition interval $[c_0, c_1]$ can be formalized as:
\begin{equation}\label{equ:define}
    \boldsymbol{h}_{\text{Stat}}(\mathbf{x},[c_0,c_1])=\operatorname{Stat}\big(\{\boldsymbol{h}(\mathbf{x},c);c\sim [c_0, c_1]\subseteq [c_\text{min}, c_\text{max}]\}\big),
\end{equation}
where $\operatorname{Stat}(\cdot)$ is the targeted \emph{statistical function} and may denote an integral, quantile, maximum, or other statistic. The above formalization is naturally compatible with existing single-condition inference conventions shown in Figure \ref{fig:intro}(a), where we can learn a neural network $\boldsymbol{h}_{\theta}\approx \boldsymbol{h}$ for single-sample regression \cite{li2021fourier,wu2024Transolver}. However, such a single-case view incurs high computation cost due to multiple inferences and cannot easily adapt to changes in the condition interval, which requires a new round of sampling.

\begin{figure}[t]
\begin{center}
\centerline{\includegraphics[width=\textwidth]{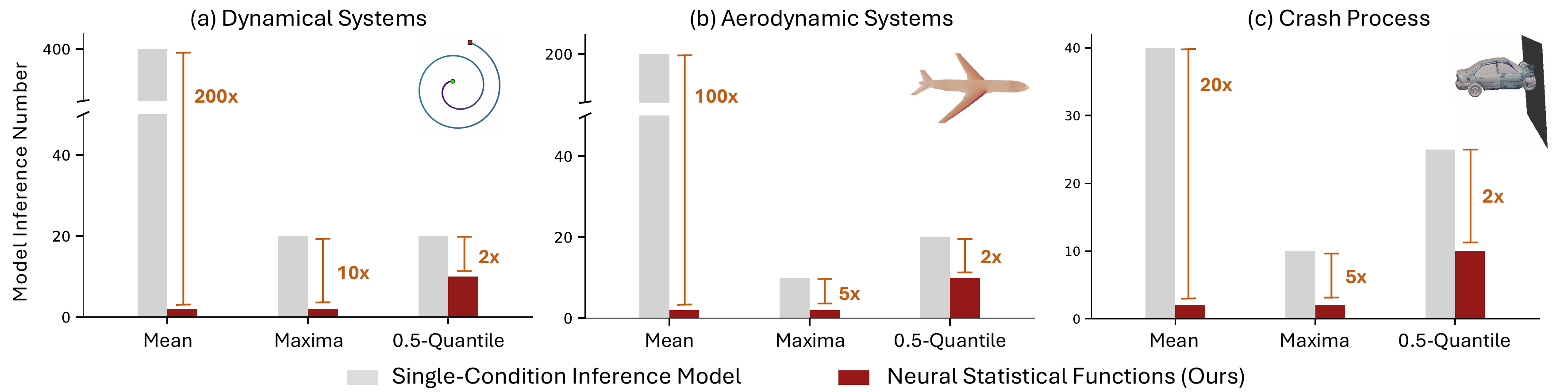}}
	\caption{{Number of model inferences to achieve comparable accuracy under three experimental scenarios. Compared to the single-condition inference model, neural statistical functions do not rely on unwieldy sampling operations, thereby significantly reducing the number of model inferences.}
    }
	\label{fig:teaser}
\end{center}
\vspace{-20pt}
\end{figure}

We depart from this convention by adopting a new perspective on the varying-condition regression task, which we call the \emph{prefix} view. As demonstrated in Figure \ref{fig:intro}(b), rather than learning the specific quantity under a single operating condition $c$, we propose to learn the neural network to approximate the \emph{prefix statistics} $\tilde{\boldsymbol{h}}(\mathbf{x},\cdot)$. Taking integral estimation as an example, let $\tilde{\boldsymbol{h}}(\mathbf{x},c)$ denote the accumulated response from $c_{\min}$ to $c$. Then the integral over any interval $[c_0,c_1]$ can be recovered by simple subtraction, which can be formally described as:
\begin{equation}\label{equ:integral}
    \boldsymbol{h}_{\text{integral}}(\mathbf{x},[c_0,c_1])=\tilde{\boldsymbol{h}}_{\text{integral}}(\mathbf{x},c_1)-\tilde{\boldsymbol{h}}_{{\text{integral}}}(\mathbf{x},c_0),\ \text{where}\ \tilde{\boldsymbol{h}}_{{\text{integral}}}(\mathbf{x},c)=\int_{c_\text{min}}^c \boldsymbol{h}(\mathbf{x}, u)\mathrm{d}u.
\end{equation}
This new parameterization frees statistics estimation from unwieldy sampling operations by leveraging the non-linear approximation capacity of neural networks to fit $\tilde{\boldsymbol{h}}(\cdot,\cdot)$, which leads to a family of deep learning models we call \emph{neural statistical functions}.

Introducing the prefix view does not, however, by itself, ensure successful training and adoption of neural statistical functions for general statistical quantities. In addition to introducing this new type of neural network, we also construct the whole training and inference pipeline. \textbf{1)} Regarding inference generalizability, although the prefix view is natural for integral, it is non-trivial to extend this idea to more complex statistics, \emph{e.g.}, quantile and maxima, whose underlying statistical functions are not additive over intervals. Through a set of principled mathematical transformations, we demonstrate that these widely used statistics can be accurately calculated or approximated based on the prefix integral for a transformed function, which enables the general use of neural statistical functions. \textbf{2)} Regarding training efficiency, rather than directly using the prefix integral computed from ground-truth data as supervision, which may also be extremely time-consuming and biased due to data sparsity, we find a principled identity that connects prefix statistics and individual-case regression model $\boldsymbol{h}(\cdot,\cdot)$, and further derive an efficient and stable learning objective.

With neural statistical functions, we achieve accurate and efficient estimation of essential statistics of diverse processes, including accumulated energy in dynamical systems, quantiles of aerodynamic responses, and maximum stress in crash, achieving up to two orders of magnitude lower latency while maintaining comparable accuracy, demonstrating its potential in complex decision-making.

\section{Related Work}

As a new type of neural network, neural statistical functions define a complete training--inference paradigm. While technically distinct, they are conceptually related to several lines of prior work.

\vspace{-5pt}
\paragraph{Integral in neural networks} Integrals appear widely in neural networks as a basic operation in model architectures \cite{chen2018neural,dupont2019augmented} and generative models \cite{ho2020denoising,lu2022dpm,song2023consistency,geng2025mean}. As a pioneering work, Neural ODE~\cite{chen2018neural} reformulates the representation learning in multilayer neural networks as a continuous-time integral, where the neural network learns the derivative of hidden states. Extending the numerical integral from model depth to real-world time, Neural ODE can be applied to simulate dynamic processes~\cite{rubanova2019latent}. Recently, diffusion models have been developed based on multi-step denoising, where the numerical integral is adopted for inference \cite{ho2020denoising,lu2022dpm}. Specifically, Mean Flow \cite{geng2025mean} proposes to directly learn the average velocity with one neural network, which does not rely on numerical integration along the denoising steps. Unlike prior work, our models focus on the statistical estimation over varying conditions and present a general framework for statistics beyond integral.

\vspace{-5pt}
\paragraph{Stochastic analysis} As a representative decision-making task involving varying circumstances, mechanical design commonly uses stochastic analysis to study system behavior under varying conditions. Classical approaches such as stochastic finite element method (FEM)~\cite{stefanou2009stochastic,langlois2016stochastic} provide principled tools for such analysis, but they can be computationally expensive despite developments in acceleration techniques. Although neural networks have been developed as efficient surrogates for numerical methods \cite{jmlr_operator,azizzadenesheli2024neural}, prior work focuses on deterministic methods, where neural networks for stochastic analysis, \emph{e.g.}, stochastic FEM, are still underexplored. Uncertainty quantification for neural networks is also related, as it typically estimates the uncertainty or variance of model predictions for a fixed input~\cite{gal2016dropout}. In contrast, this paper focuses on statistical properties under varying conditions, rather than uncertainty at a single condition.

\section{Neural Statistical Functions}

As formalized in Eq.~\eqref{equ:define}, given $\operatorname{Stat}(\cdot)$ refers to the type of statistical functions and $\boldsymbol{h}(\mathbf{x},c)$ denotes the target quantity on position $\mathbf{x}$ under condition $c$, our goal is to predict the statistics $\boldsymbol{h}_{\text{Stat}}(\mathbf{x},[c_0,c_1])$ corresponding to the condition interval $[c_0,c_1]\subseteq [c_{\min},c_{\max}]\subset\R$. Here, we assume that varying conditions can be projected to a one-dimensional space and users may query different intervals.

Next, we will first extend the prefix view in Eq.~\eqref{equ:integral} to general statistics inference, such as quantile and maxima, where a \emph{prefix statistics framework} is presented. Then, we will derive the training objective of neural statistical functions, which is based on an identity principle of prefix statistics. For main text clarity, we defer all the proofs of propositions and theorems to the Appendix \ref{appdix:proof}.

\vspace{-5pt}
\paragraph{Notation simplification} For clarity, we reparameterize the scalar uncertainty in quantile space. Let
\begin{equation}\label{equ:condition_project}
    s = \cdf(c)\in[0,1],
    \qquad
    c = \icdf(s),
\end{equation}
where $\cdf$ is the cumulative distribution function of the condition $c$'s distribution $\mu$ and $\icdf$ is its inverse quantile map. Under this change of variables, any interval $[c_0,c_1]$ induces a quantile interval
\begin{equation}
    s_0 := \cdf(c_0),
    \qquad
    s_1 := \cdf(c_1),
    \qquad
    m := s_1-s_0.
    \label{eq:interval_mass}
\end{equation}
The quantity $m$ is the probability mass of the interval under $\mu$. Quantile space is a convenient coordinate system because all distributions are mapped to the common domain $[0,1]$, and conditional integrals over $[c_0,c_1]$ become normalized integrals over $[s_0,s_1]$. This allows the learning problem to be expressed in a distribution-normalized way that is independent of the original scale of $c$.

\subsection{Prefix Statistics Framework} 

\begin{wraptable}{r}{0.4\textwidth}
\vspace{-12pt}
\centering
\caption{Unified prefix-statistic framework. Each statistic is defined by a task-specific transformed local signal $\boldsymbol{\psi}$.}
\vspace{-2pt}
\label{tab:framework}
\small
\setlength{\tabcolsep}{5pt}
\begin{tabular}{lcc}
\toprule
{$\boldsymbol{h}_{\text{Stat}}(\mathbf{x},[c_0,c_1])$}  
& {$\boldsymbol{\psi}(\mathbf{x},s)$} 
 \\
\midrule
Mean 
& $\boldsymbol{h}(\mathbf{x},\icdf(s))$ \\
Max 
& $\exp\bigl(\beta\,\boldsymbol{h}(\mathbf{x},\icdf(s))\bigr)$ \\
Quantile (CDF) 
& $\1\!\left\{\boldsymbol{h}(\mathbf{x},\icdf(s)) \le \mathbf{y} \right\}$ \\
\bottomrule
\end{tabular}
\vspace{-10pt}
\end{wraptable}

To support general statistical functions, we built a framework based on the transformed local signal $\boldsymbol{\psi}(\mathbf{x},s)$, whose choice depends on the target statistical function (Table \ref{tab:framework}). 

For a given $\boldsymbol{\psi}$, define its prefix statistics based on integral
\begin{equation}\label{equ:pre_stat}
    \tilde{\boldsymbol{h}}(\mathbf{x},s)
    :=
    \frac{1}{s}\int_0^s \boldsymbol{\psi}(\mathbf{x},t)\,\dd t,
    \qquad s\in(0,1].
\end{equation}
This prefix-statistics construction defines a broader amortizable framework for interval-conditioned prediction. In particular, whenever a target statistic can be expressed as a function of a normalized interval integral,
\begin{equation}\label{equ:infer}
    \boldsymbol{h}_{\text{Stat}}(\mathbf{x},[c_0,c_1])=
\boldsymbol{\Phi}\!\left(
\frac{1}{m}\int_{s_0}^{s_1}\boldsymbol{\psi}(\mathbf{x},s)\,\dd s
\right)=\boldsymbol{\Phi}\!\left(\frac{s_1 \tilde{\boldsymbol{h}}(\mathbf{x},s_1)-s_0 \tilde{\boldsymbol{h}}(\mathbf{x},s_0)}{m}\right),
\ m:=s_1-s_0.
\end{equation}
Thus, the proposed framework is a reusable primitive for interval-conditioned prediction. The interval mean and maximum are direct instances with different choices of the post-processing map $\boldsymbol{\Phi}$, while the quantile case fits this template at the level of the interval cumulative distribution, followed by inversion in the response variable. Next, we list the implementation for several typical statistics.

\underline{\emph{(i) Mean.}} Similar to the usage described in Eq.~\eqref{equ:integral}, mean can be calculated by defining $\boldsymbol{\Phi}$ as identity.

\underline{\emph{(ii) Maxima.}} We adopt a smooth log-sum-exp surrogate to approximate the maxima function:
\begin{equation}\label{eq:max_soft_definition}
\boldsymbol{h}_{\text{Stat}}^\beta(\mathbf{x},[c_0,c_1])
    :=
    \frac{1}{\beta}
    \log\left(
    \frac{1}{m}
    \int_{s_0}^{s_1}
    \exp\bigl(\beta\,h(\mathbf{x},\icdf(s))\bigr)\,\dd s
    \right),
\end{equation}
with temperature parameter $\beta>0$, namely $\boldsymbol{\Phi}(\cdot)=\frac{1}{\beta}\log(\cdot)$. This is a continuous version of entropy-based
smoothing for max-type objectives, which as $\beta\to\infty$, the quantity converges to the interval supremum under mild regularity conditions \citep{nesterov2005smooth}.

\underline{\emph{(iii) $\alpha$-Quantile.}} Here, we consider $\alpha$-quantiles of different condition intervals, where $\alpha\in(0,1)$ be a target probability level. Specifically, the interval-conditioned $\alpha$-quantile of the response is defined as the value $\boldsymbol{h}_{\text{Stat}}^\alpha(\mathbf{x},[c_0,c_1])$, which satisfies the following equality:
\begin{equation}\label{equ:quantile_definition}
    \mathbb{P}\!\left(
    \boldsymbol{h}(\mathbf{x},c)\le \boldsymbol{h}_{\text{Stat}}^\alpha(\mathbf{x},[c_0,c_1])
    \,\middle|\,
    c\in[c_0,c_1]
    \right)
    =
    \alpha.
\end{equation}
Towards a unified framework, we work with the conditional cumulative distribution function (CDF) over response values $\boldsymbol{h}(\cdot,\cdot)$. Let $\mathbf{y}$ denote a candidate point-wise response level and define
\begin{equation}\label{equ:cdf}
    \boldsymbol{\psi}(\mathbf{x},s;\mathbf{y})
    :=
    \1\!\left\{
    \boldsymbol{h}(\mathbf{x},\icdf(s)) \le \mathbf{y}
    \right\},\ \tilde{\boldsymbol{h}}(\mathbf{x},s;\mathbf{y})
    :=
    \frac{1}{s}\int_0^s
    \1\!\left\{
    \boldsymbol{h}(\mathbf{x},\icdf(t)) \le \mathbf{y}
    \right\}\,\dd t.
\end{equation}
Then, the interval-conditioned $\alpha$-quantile can be obtained by inversion:
\begin{equation}\label{equ:quantile}
    \boldsymbol{h}_{\text{Stat}}^\alpha(\mathbf{x},[c_0,c_1])
    :=
    \inf\left\{
    \mathbf{y}:\,
    \frac{1}{m}\int_{s_0}^{s_1}\1\!\left\{
    \boldsymbol{h}(\mathbf{x},\icdf(t)) \le \mathbf{y}
    \right\}\,\dd t\ge \alpha
    \right\}.
\end{equation}
Note that unlike mean and maxima that only need two-time inference of prefix statistic $\tilde{\boldsymbol{h}}$ as described in Eq.~\eqref{equ:infer}, $\alpha$-quantile actually requires multiple inference under different $\mathbf{y}$. Fortunately, since the transformed function $\boldsymbol{\psi}$ is monotonically non-decreasing, we can adopt the classical binary search~\cite{nowak2008generalized} for speedup. Specifically, given $\epsilon$ as the tolerable error, we select $\lceil \frac{\mathbf{y}_{\text{max}}-\mathbf{y}_{\text{min}}}{\epsilon}\rceil$ candidates for $\mathbf{y}$ uniformly-spaced within global minima $\mathbf{y}_{\text{min}}$ and maxima $\mathbf{y}_{\text{max}}$, the complexity is $\mathcal{O}(\log \lceil \frac{\mathbf{y}_{\text{max}}-\mathbf{y}_{\text{min}}}{\epsilon}\rceil)$ under binary search, where around ten evaluations of $\tilde{\boldsymbol{h}}$ are sufficient when
$\frac{(\mathbf{y}_{\text{max}}-\mathbf{y}_{\text{min}})}{\epsilon} \approx 10^3$.

\begin{remark}[Generality of framework]\label{remark:generality}
Beyond $\alpha$-quantile, Eq.~\eqref{equ:cdf} actually learns the interval-conditioned CDF of the response distribution, which can serve as a general primitive for a broad range of distributional statistics. For example,
the probability that the response lies within a target range $[\mathbf{a},\mathbf{b}]$
can be directly computed as follows
\begin{equation}
\begin{aligned}
    \mathbb{P}\!\left(
    \mathbf{a}<\boldsymbol{h}(\mathbf{x},c)\le \mathbf{b}
    \,\middle|\,
    c\in[c_0,c_1]
    \right)
    =
    \frac{
    s_1\tilde{\boldsymbol{h}}(\mathbf{x},s_1;\mathbf{b})
    -
    s_0\tilde{\boldsymbol{h}}(\mathbf{x},s_0;\mathbf{b})
    }{m}
    -
    \frac{
    s_1\tilde{\boldsymbol{h}}(\mathbf{x},s_1;\mathbf{a})
    -
    s_0\tilde{\boldsymbol{h}}(\mathbf{x},s_0;\mathbf{a})
    }{m}.
\end{aligned}
\end{equation}
Similarly, the probability of exceeding a critical threshold $\boldsymbol{\tau}$ is given by
\begin{equation}
    \mathbb{P}\!\left(
    \boldsymbol{h}(\mathbf{x},c)>\boldsymbol{\tau}
    \,\middle|\,
    c\in[c_0,c_1]
    \right)
    =
    1-
    \frac{
    s_1\tilde{\boldsymbol{h}}(\mathbf{x},s_1;\boldsymbol{\tau})
    -
    s_0\tilde{\boldsymbol{h}}(\mathbf{x},s_0;\boldsymbol{\tau})
    }{m}.
\end{equation}
Therefore, although quantile inference requires an additional inversion step,
the learned conditional CDF provides a reusable distributional representation
for many downstream risk-aware statistics.
\end{remark}

\subsection{Model Training}

By introducing prefix statistics (Eq.~\eqref{equ:pre_stat}), we can set the statistics estimation free from computation-intensive sampling. Next, we will detail how to train a neural surrogate for $\tilde{\boldsymbol{h}}(\cdot,\cdot)$.

\vspace{-5pt}
\paragraph{Prefix statistics identity} Differentiating Eq.~\eqref{equ:pre_stat} with respect to $s$ yields the first-order identity
\begin{equation}\label{eq:running_identity}
    s\,\partial_s \tilde{\boldsymbol{h}}(\mathbf{x},s) + \tilde{\boldsymbol{h}}(\mathbf{x},s)
    =
    \boldsymbol{\psi}(\mathbf{x},s),
\end{equation}
which directly relates the prefix statistic to the instantaneous solver response. It converts an interval-conditioned aggregation problem into a local differential constraint in quantiled condition space.

\begin{proposition}[Equivalence of first order identity and prefix statistics]
\label{prop:ode_prefix_equivalence}
Fix $\mathbf{x}$ and suppose $\boldsymbol{\psi}(\mathbf{x},\cdot)\in L^1(0,1]$. Let
$\tilde{\boldsymbol{h}}(\mathbf{x},\cdot)$ be differentiable on $(0,1]$. Then the following are equivalent:
\begin{enumerate}[leftmargin=1.5em]
    \item $\tilde{\boldsymbol{h}}$ satisfies Eq.~\eqref{eq:running_identity} a.e.~$s\in(0,1]$ 
    together with the boundary condition $\lim_{s\to0} s\,\tilde{\boldsymbol{h}}(\mathbf{x},s)=0$;
    \item $\tilde{\boldsymbol{h}}$ admits the prefix-integral representation in Eq.~\eqref{equ:pre_stat}.
\end{enumerate}
\end{proposition}
Proposition~\ref{prop:ode_prefix_equivalence} demonstrates that the first-order identity is not merely a property of the prefix statistic in Eq.~\eqref{equ:pre_stat}; together with the natural boundary condition at $s\to 0$, it uniquely characterizes the desired prefix statistic itself. Thus, learning a function that satisfies Eq.~\eqref{eq:running_identity} is equivalent to learning the targeted prefix-integral representation needed for interval statistics inference.

\vspace{-5pt}
\paragraph{Training objective} We parameterize the prefix statistic with a neural surrogate $\tilde{\boldsymbol{h}}_\theta(\mathbf{x},s)$. Then the first-order identity formalized in Eq.~\eqref{eq:running_identity} naturally derives a learning objective:
\begin{equation}\label{equ:generic_target}
    \mathcal{L}(\theta)
    =
    \E_{s\sim \mathrm{Unif}(0,1]}\left[
    \left\|
    \tilde{\boldsymbol{h}}_\theta(\mathbf{x},s)
    -
    \sg\!\left(\tilde{\boldsymbol{h}}_{\mathrm{tgt}}(\mathbf{x},s)\right)
    \right\|_2^2
    \right],\ \text{where}\ 
    \tilde{\boldsymbol{h}}_{\mathrm{tgt}}(\mathbf{x},s)
    =
    \boldsymbol{\psi}(\mathbf{x},s)
    -
    s\,\partial_s \tilde{\boldsymbol{h}}_\theta(\mathbf{x},s).
\end{equation}
Here $\sg(\cdot)$ denotes stop-gradient. As discussed above, this formulation enforces the first-order identity while avoiding backpropagation through the derivative-dependent target. Note that each training sample requires only a local solver evaluation at $c=\icdf(s)$. The network learns a global family of prefix statistics over quantile space, rather than separate predictors for separate intervals.

\begin{theorem}[Interval statistics inference error]
\label{thm:residual_query_error}
Fix $\mathbf{x}$ and let $\tilde{\boldsymbol{h}}_\theta(\mathbf{x},s)$ be differentiable w.r.t~$s$. Define the residual
    $\boldsymbol{\delta}_\theta(\mathbf{x},s)
    :=
    s\,\partial_s \tilde{\boldsymbol{h}}_\theta(\mathbf{x},s)
    +
    \tilde{\boldsymbol{h}}_\theta(\mathbf{x},s)
    -
    \boldsymbol{\psi}(\mathbf{x},s)$.
Then for any interval $0 \le s_0 < s_1 \le 1$,
\begin{equation}
    \left\|
    \frac{
    s_1 \tilde{\boldsymbol{h}}_\theta(\mathbf{x},s_1)
    -
    s_0 \tilde{\boldsymbol{h}}_\theta(\mathbf{x},s_0)
    }{m}
    -
    \frac{1}{m}\int_{s_0}^{s_1}\boldsymbol{\psi}(\mathbf{x},s)\,\dd s
    \right\|
    \le
    \frac{1}{m}\int_{s_0}^{s_1}
    \left\|
    \boldsymbol{\delta}_\theta(\mathbf{x},s)
    \right\|\,\dd s.
    \label{eq:normalized_residual_error_bound}
\end{equation}
\end{theorem}

Theorem~\ref{thm:residual_query_error} shows that controlling the local residual of Eq.~\eqref{eq:running_identity} is sufficient to control the error of estimated interval queries that calculated from endpoint evaluations of $\tilde{\boldsymbol{h}}_\theta$ as formalized in Eq.~\eqref{equ:infer}.

\begin{proposition}[Narrow-interval sensitivity]
\label{prop:narrow_interval_sensitivity}
Fix $\mathbf{x}$ and let $
\boldsymbol{\epsilon}_\theta(\mathbf{x},s):=\tilde{\boldsymbol{h}}_\theta(\mathbf{x},s)-\tilde{\boldsymbol{h}}(\mathbf{x},s)$,
where $\tilde{\boldsymbol{h}}$ denotes the true prefix statistics. For any interval and $s^\star\in(0,1)$, if $\epsilon_\theta(\mathbf{x},\cdot)$ is differentiable, then
\begin{equation}
\lim_{\substack{s_0,s_1\to s^\star, s_1>s_0}}
\frac{
s_1 \tilde{\boldsymbol{h}}_\theta(\mathbf{x},s_1)-s_0 \tilde{\boldsymbol{h}}_\theta(\mathbf{x},s_0)
}{s_1-s_0}
-
\frac{
s_1 \tilde{\boldsymbol{h}}(\mathbf{x},s_1)-s_0 \tilde{\boldsymbol{h}}(\mathbf{x},s_0)
}{s_1-s_0}
=
\boldsymbol{\epsilon}_\theta(\mathbf{x},s^\star)
+
s^\star \partial_s\boldsymbol{\epsilon}_\theta(\mathbf{x},s^\star).
\label{eq:narrow_interval_limit}
\end{equation}

\begin{figure}[t]
\begin{algorithm}[H]
\caption{Training pipeline of neural statistical functions $\tilde{\boldsymbol{h}}_\theta$.}
\label{alg:hybrid_training}
\begin{algorithmic}[1]
\Require Dataset $\mathcal{D}=\{(\mathbf{x},c,\mathbf{y})\}$, single-sample predictor $\boldsymbol{h}_{\theta_{\text{sample}}}$, finite-difference step $\Delta s$, data loss weight $\lambda_{\text{data}}=0.1$
\For{each training step}
    \State Sample a design $(\mathbf{x},c,\mathbf{y})\sim\mathcal{D}$
    \State Sample $s_{\text{neural}}\sim\mathrm{Unif}(0,1]$ and set 
    $\boldsymbol{\psi}_{\text{neural}}\leftarrow
    \boldsymbol{\psi}\!\left(\boldsymbol{h}_{\theta_{\text{sample}}}(\mathbf{x},\icdf(s_{\text{neural}}))\right)$
    \Comment{\emph{Neural branch}}
    \State Set 
    $s_{\text{data}}\leftarrow\cdf(c)$ and 
    $\boldsymbol{\psi}_{\text{data}}\leftarrow\boldsymbol{\psi}(\mathbf{y})$
    \Comment{\emph{Data branch\ \ \ }}

    \State Define 
    $\mathcal{\tilde{\boldsymbol{h}}_{\mathrm{tgt}}}(\mathbf{x},s,\boldsymbol{\psi})
    =
    \boldsymbol{\psi}
    -
    s\,\partial_s\tilde{\boldsymbol{h}}_\theta(\mathbf{x},s)$ with $\partial_s\tilde{\boldsymbol{h}}_\theta$ estimated by $\Delta s$-step finite differences
    
    \State Compute neural-predictor-based loss\[
    \mathcal{L}_{\text{neural}}
    =
    \big\|
    \tilde{\boldsymbol{h}}_\theta(\mathbf{x},s_{\text{neural}})
    -
    \sg\big(\tilde{\boldsymbol{h}}_{\mathrm{tgt}}(\mathbf{x},s_{\text{neural}},\boldsymbol{\psi}_{\text{neural}})\big)
    \big\|_2^2
    \]
    \State Compute simulation-data-based loss
    \[
    \mathcal{L}_{\text{data}}
    =
    \big\|
    \tilde{\boldsymbol{h}}_\theta(\mathbf{x},s_{\text{data}})
    -
    \sg\big(\tilde{\boldsymbol{h}}_{\mathrm{tgt}}(\mathbf{x},s_{\text{data}},\boldsymbol{\psi}_{\text{data}})\big)
    \big\|_2^2
    \]

    \State Take a gradient step on
    $\mathcal{L}=\mathcal{L}_{\text{neural}}+\lambda_{\text{data}}\mathcal{L}_{\text{data}}$
\EndFor
\end{algorithmic}
\end{algorithm}
\vspace{-20pt}
\end{figure}

\end{proposition}
The above proposition shows that small pointwise
approximation error in the prefix statistic does not by itself guarantee accurate reconstruction on narrow intervals. Indeed, the statistics inference error depends on how the prefix statistics error varies across the interval, and in the vanishing-width
limit it approaches $
\boldsymbol{\epsilon}_\theta(\mathbf{x},s)
+
s\,\partial_s\boldsymbol{\epsilon}_\theta(\mathbf{x},s)$.
Equivalently, since the true $\tilde{\boldsymbol{h}}$ satisfies
Eq.~\eqref{eq:running_identity}, we have
\begin{equation}
\boldsymbol{\delta}_\theta(\mathbf{x},s)
=
s\,\partial_s \tilde{\boldsymbol{h}}_\theta(\mathbf{x},s)
+
\tilde{\boldsymbol{h}}_\theta(\mathbf{x},s)
-
\boldsymbol{\psi}(\mathbf{x},s)
=
\boldsymbol{\epsilon}_\theta(\mathbf{x},s)
+
s\,\partial_s\boldsymbol{\epsilon}_\theta(\mathbf{x},s).
\label{eq:residual_as_representation_error}
\end{equation}
Thus, the gradient-matching objective in Eq.~\eqref{equ:generic_target} directly governs narrow-interval statistics.

\vspace{-5pt}
\paragraph{Hybrid supervision $\tilde{\boldsymbol{h}}_{\mathrm{tgt}}$}
The gradient-matching objective in Eq.~\eqref{equ:generic_target} only requires evaluating the locally transformed signal $\boldsymbol{\psi}(\mathbf{x},s)$ at a single quantile point $s\in(0,1]$ per training step. As demonstrated in Algorithm~\ref{alg:hybrid_training}, we adopt a hybrid scheme for practical training, which includes data and neural predictor branches. Specifically, the data supervision directly assigns the ground truth quantity $\mathbf{y}$ as the input to $\boldsymbol{\psi}$ and the neural predictor branch relies on a pretrained single-condition predictor to generate the corresponding estimate. The two regimes are complementary: the pretrained predictor provides dense, continuous coverage of $s$ but inherits predictor bias, while data supervision is unbiased at the observed samples but restricted to the discrete $c_i$ in the dataset. We therefore optimize both losses jointly, using the neural-predictor loss to provide continuous coverage and the data loss, weighted by $\lambda_{\mathrm{data}}$, to anchor training to ground-truth simulations, i.e., $\mathcal{L}=\mathcal{L}_{\mathrm{neural}}+\lambda_{\mathrm{data}}\mathcal{L}_{\mathrm{data}}$.

For the gradient term $\partial_s \tilde{\boldsymbol{h}}_\theta(\mathbf{x},s)$ in the identity loss, we use finite-difference approximation for efficiency, which can avoid the computation-intensive backpropagation process.

\section{Experiments}

\begin{wraptable}{r}{0.5\textwidth}
\vspace{-12pt}
\centering
\caption{Summary of main text experiments.}
\vspace{-2pt}
\label{tab:tasks}
\small
\setlength{\tabcolsep}{3pt}
\begin{tabular}{lcc}
\toprule
Systems
& Targeted Statistics & Varied Condition
 \\
\midrule
2D Dynamics 
& Integral energy & Time \\
Airplane Aero 
& Quantile pressure & Angle of attack \\
Car Crash 
& Maximum stress & Impact angle \\
\bottomrule
\end{tabular}
\vspace{-10pt}
\end{wraptable}

Using the general prefix statistics framework, neural statistical functions support diverse statistical estimation across systems, following the configuration in Table~\ref{tab:framework}, with overall performance summarized in Figure~\ref{fig:teaser}. Due to the context limitation, we present three representative tasks in the main text as listed in Table \ref{tab:tasks}, including temporal integral of energy of 2D dynamics, quantile pressure of airplane surface pressure under different angles of attack and maximum stress of inside elements of a car under various crash angles. Full results can be found in the Appendix \ref{appdix:full_stat}.

\vspace{-5pt}
\paragraph{Implementation details} All experiments are conducted in PyTorch~\citep{Paszke2019PyTorchAI} on NVIDIA V100 32GB GPUs. We train neural statistical functions for 2,000 epochs with AdamW~\citep{loshchilov2017fixing} and OneCycleLR~\citep{smith2019super}. They share the architecture of the corresponding single-condition inference  models: an MLP for 2D dynamics~\cite{chen2018neural} and Transolver~\cite{wu2024Transolver} for aerodynamics and crash. MC experiments are repeated over 5 seeds; curves show the mean and shaded bands indicate one standard error.

\vspace{-5pt}
\paragraph{Ground-truth statistics.}
To evaluate model performance, we construct reference statistics using the following two approaches.
\emph{(i)} Dense inference with single-condition predictors. Specifically, we evaluate the single-condition model 1,000 times, with the condition $c$ sampled from the considered interval. This provides a strong reference for statistical estimation based on single-condition models.
\emph{(ii)} Dense computation with numerical solvers. When available, we run numerical simulations~\cite{openradioss,bekemeyer2025introduction} under different conditions and compute statistics from the resulting dense sweeps. Specifically, for the 2D dynamics task, we directly compute the analytic statistics from classical mechanics. For the car crash task, we simulate 280 samples with uniformly sampled impact angles and compute statistics. For the aerodynamics task~\cite{bekemeyer2025introduction}, since the simulation configuration is not released, we only use approach~(i) as the reference. Approach~(ii) requires hundreds of numerical simulations for each design, which is computationally prohibitive in practical deployment.

\subsection{Mean/Integral of Dynamical Systems}\label{sec:dyn}

\paragraph{Setup}
\begin{wrapfigure}{r}{0.36\textwidth}
    \vspace{-40pt}
    \centering
    \includegraphics[width=0.25\textwidth]{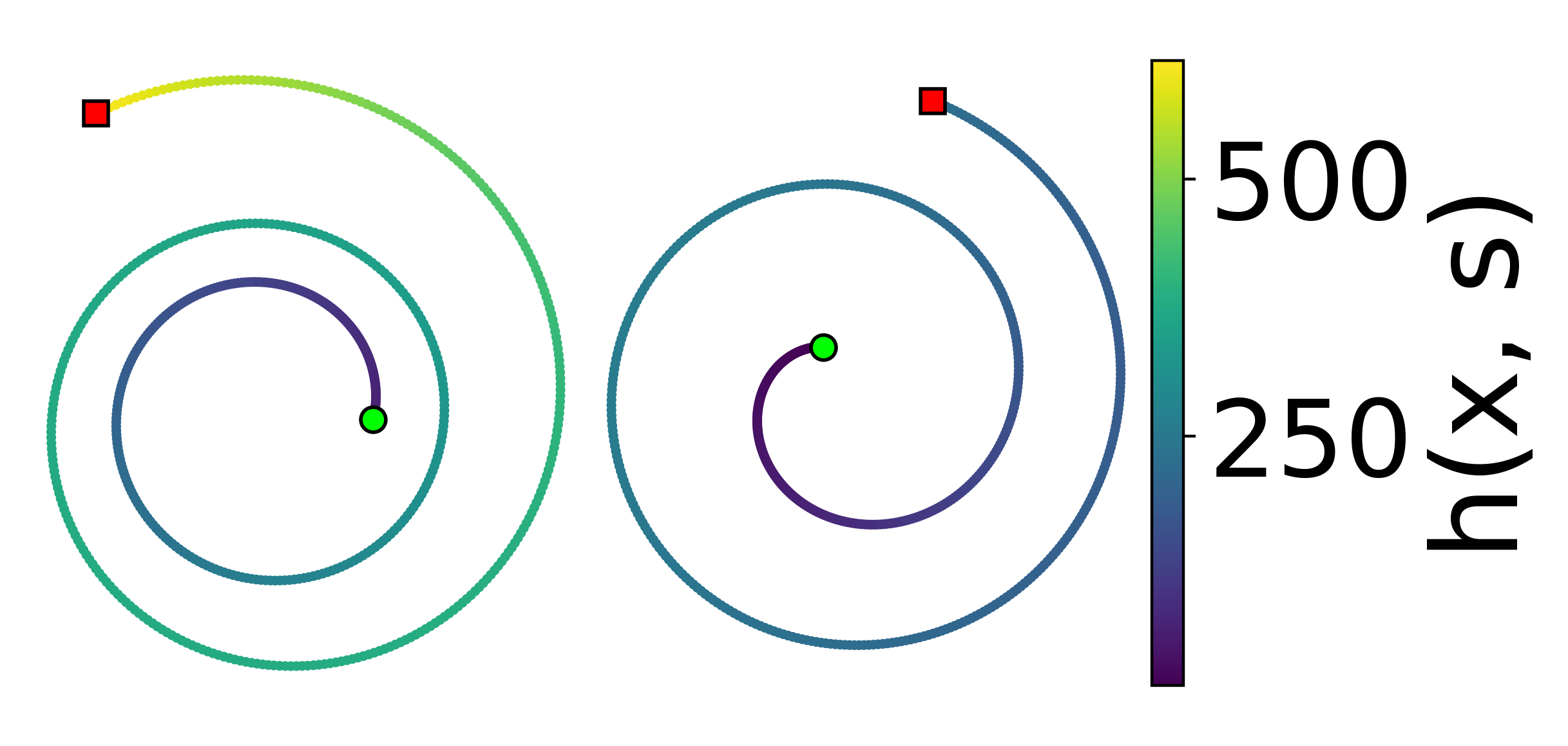}
    \caption{Example trajectories.}
    \label{fig:spiral_sample}
    \vspace{-10pt}
\end{wrapfigure}
We evaluate interval-conditioned mean prediction on a synthetic 2D dynamical system parameterized by a normalized timestamp $s\in[0,1]$. This benchmark contains 2,500 data samples~\cite{chen2018neural}. Each trajectory (Figure \ref{fig:spiral_sample}) is specified by a $7$-dimensional parameter vector $\mathbf{x}=(a,b,\phi_r,\omega_0,\alpha,\phi_\omega,\phi_0)$ describing the system state, which is drawn i.i.d.~at trajectory construction, defining a planar curve $\mathbf{y}(s)=r(s)\bigl(\cos\eta(s),\sin\eta(s)\bigr)$ with radius $r(s)=a+bs+0.1\sin(2\pi s+\phi_r)$ and angular speed $\eta'(s)=\omega_0+2\pi\alpha\cos(2\pi s+\phi_\omega)$. We adopt 2,000 samples for training and the remaining 500 samples with different system parameters for testing. For each case, we consider the mean of instantaneous kinetic-energy, which is $\boldsymbol{h}_{\text{mean}}(\mathbf{x},[s_0, s_1])=\frac{1}{s_1-s_0}\int_{s_0}^{s_1}\boldsymbol{h}(\mathbf{x},s)\dd s=\frac{1}{s_1-s_0}\int_{s_0}^{s_1}r'(s)^2+\bigl(r(s)\eta'(s)\bigr)^2 \dd s$. 

Because the instantaneous quantity $\boldsymbol{h}(\mathbf{x},\cdot)$ admits a closed form, the interval mean can be evaluated analytically. We use this analytic interval mean as the ground-truth reference when reporting absolute accuracy. Besides, we also compute a dense 1,000-sample Monte Carlo estimate based on the single-sample model for each interval, which serves as a high-sample reference for comparison. Both single-sample and neural statistical function $\tilde{\boldsymbol{h}}_\theta$ are instantiated as a $4$-layer SiLU MLP \cite{ramachandran2017searching,shazeer2020glu} with hidden width $256$ and input $\mathbf{x}\oplus s$. For baselines, we compare against limited inference of single-sample model, namely $\frac{1}{K}\sum_{k=1}^{K}\boldsymbol{h}_{\theta_{\text{sample}}}(\mathbf{x},s_k)$ with $s_k\sim\mathrm{Unif}(s_0,s_1), 10\le K \le 400$.

\vspace{-5pt}
\paragraph{$\boldsymbol{\psi}$ implementation} Here $c=s$ is uniform on $[0,1]$, so $\icdf(\cdot)$ is the identity and $\boldsymbol{\psi}(\mathbf{x},s)=\boldsymbol{h}(\mathbf{x},s)$ for mean. Models are tested over intervals $[s_0,s_1]\subseteq[0,1]$ with widths $(s_1-s_0)\in[0.5,\,0.95]$.

\begin{figure}
    \centering
    \includegraphics[width=\linewidth]{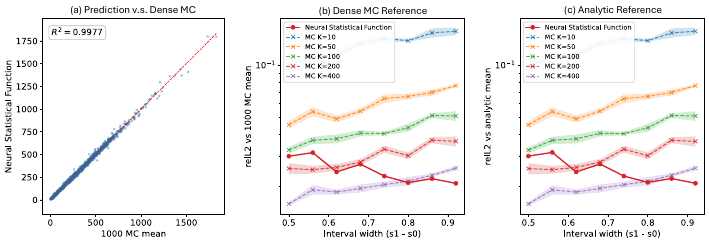}
    \vspace{-10pt}
        \caption{Mean estimation on 2D dynamical systems. 
    (a) Parity plot comparing neural statistical functions' predictions with the $1{,}000$-sample Monte Carlo (MC) reference mean over all test trajectory--interval pairs. 
    (b) Width-binned relative $\ell_2$ error of the neural statistical functions and Monte Carlo baselines, computed with respect to the 1000-sample MC reference mean. 
    (c) Width-binned relative $\ell_2$ error of the same predictions, computed with respect to the analytic interval mean.}
    \label{fig:spiral_mean}
    \vspace{-10pt}
\end{figure}

\vspace{-5pt}
\paragraph{Results}Across both evaluation references, neural statistical functions achieve a clear accuracy-efficiency advantage over Monte Carlo estimation. Using only two endpoint evaluations of $\tilde{\boldsymbol{h}}_\theta$, the proposed method consistently attains lower relative error than finite-sample Monte Carlo baselines, and outperforms even the $K=400$ estimator on wide intervals. This advantage becomes more pronounced as the interval width increases, where the accumulated energy depends on a broader and more complex portion of the trajectory and Monte Carlo estimators require many more samples to reduce variance. In contrast, small-$K$ Monte Carlo estimates remain noticeably noisy, and their accuracy improves only gradually with increasing sample count. The similar trends obtained when evaluating against the dense $1{,}000$-sample Monte Carlo reference and the analytic interval mean in Figures~\ref{fig:spiral_mean}(b,c) further indicate that the dense Monte Carlo reference is sufficiently accurate for practical comparison. More detailed results can be found in Appendix~\ref{app:spiral_diagnostics}.

\subsection{$\alpha$-Quantile of Aerodynamic Systems}\label{sec:aero_sys}

\paragraph{Setup} 
\begin{wrapfigure}{r}{0.4\textwidth}
    \vspace{-35pt}
    \centering
    \includegraphics[width=0.38\textwidth]{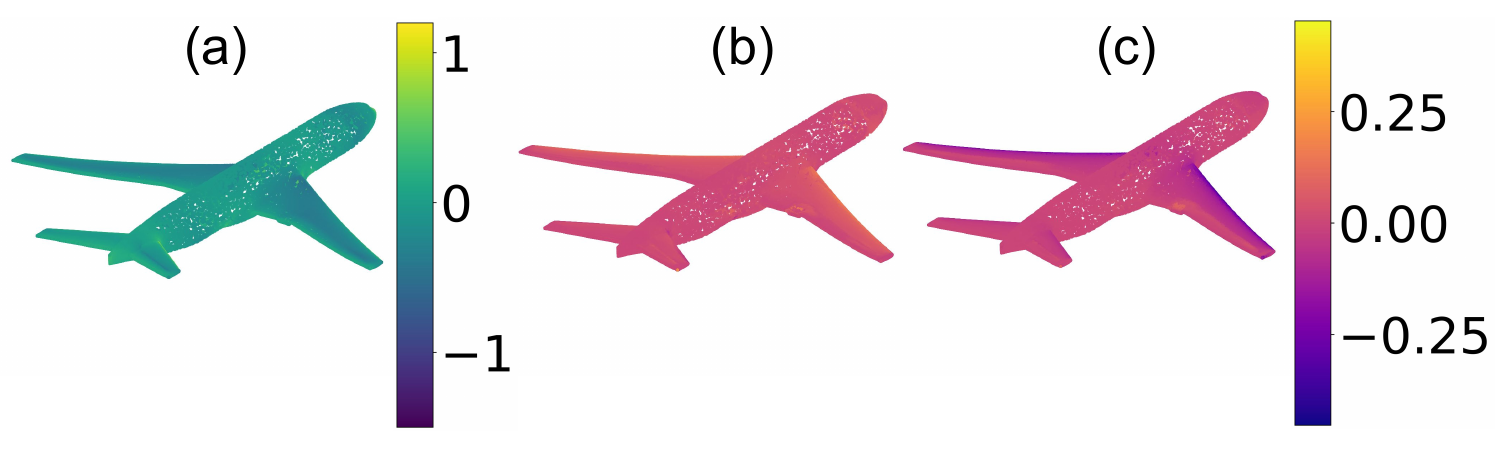}
    \vspace{-3pt}
    \caption{0.9-quantile of point-wise pressure over a 9.8$^\circ$ AoA interval: (a) dense-MC reference, (b) neural statistical function error, and (c) MC ($K=10$) error.}
    \label{fig:aero_sample}
    \vspace{-10pt}
\end{wrapfigure}We evaluate interval-conditioned quantile prediction on the NASA-CRM \cite{bekemeyer2025introduction}, which provides high-quality simulations for flying airplanes. This benchmark contains 149 samples, where each sample is simulated under various angles of attack (AoA) and geometries. The quantity of interest $\boldsymbol{h}(\mathbf{x},c)$ is the point-wise pressure coefficient on the airplane surface $\mathbf{x}$. The uncertain parameter $c$ is AoA over $[c_{\text{min}},c_{\text{max}}]=[-2.5^\circ,7.5^\circ]$, treated as
uniform, and queries are defined over intervals $[c_0,c_1]$ of varying width. We adopt 105 samples for training and the remaining 44 samples with varied geometries for testing. The targeted statistic is the conditional $\alpha$-quantile (Eq.~\eqref{equ:quantile_definition}) for each geometry.

As discussed above, since we cannot access the simulation configuration of this dataset, we adopt a pre-trained Transolver model \cite{wu2024Transolver} with 200-time dense inference to approximate the ground truth statistics. Here Transolver is an established neural simulator, which can produce high-accuracy point-wise simulation. We use the same architecture in training neural statistical function $\tilde{\boldsymbol{h}}_\theta$. 

\vspace{-5pt}
\paragraph{$\boldsymbol{\psi}$ implementation} As shown in Eq.~\eqref{equ:cdf}-\eqref{equ:quantile}, $\boldsymbol{\psi}$ is defined as an indicator function, which is unstable in calculating differentiation. Thus, we smooth $\boldsymbol{\psi}$ with a sigmoid-style approximation:
\begin{equation}
    \boldsymbol{\psi}(\mathbf{x},s;\mathbf{y})=\big(1+\exp\big(-\frac{\mathbf{y}-\boldsymbol{h}(\mathbf{x},\icdf(s))}{\varepsilon}\big)\big)^{-1}\approx \1\!\left\{
    \boldsymbol{h}(\mathbf{x},\icdf(s)) \le \mathbf{y}
    \right\},
\end{equation} and the smoothing width $\varepsilon$ annealed from $0.1$ to $0.01$ during training, in units of one standard deviation of the normalized surface pressure response. In practice, we conduct a five-step binary search for each interval conditioned inference, which involves a 10-time model inference.

\begin{figure}
    \centering

    \includegraphics[width=\linewidth]{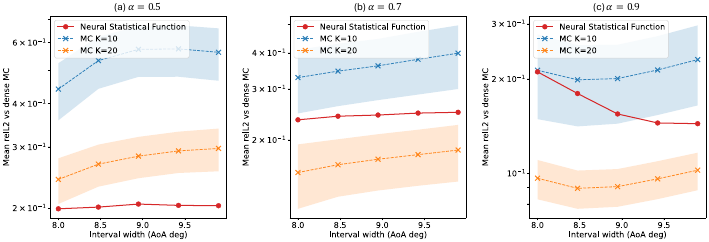}
    \vspace{-10pt}

\caption{Width-binned relative $\ell_2$ error for interval-conditioned pressure quantile prediction on NASA-CRM:
(a) $\alpha=0.5$, (b) $\alpha=0.7$, and (c) $\alpha=0.9$.
Errors are computed against the dense Transolver reference; red denotes neural statistical function and dashed curves denote MC baselines.
}
    \vspace{-10pt}

    \label{fig:aero_quantile}
\end{figure}

\vspace{-5pt}
\paragraph{Results}
As shown in Figure~\ref{fig:aero_quantile}, neural statistical functions consistently improve over the MC baseline with $K=10$, and this advantage becomes more evident on wider AoA intervals, where the queried statistic depends on a broader range of aerodynamic responses and finite-sample MC estimates exhibit larger variance. It is also notable that quantile prediction is much more challenging than mean estimation because the model is expected to learn an interval-conditioned CDF, while our method still brings relative improvement under this task, further verifying the capability of neural statistical functions. Although the $K=20$ benefits from additional samples and remains competitive for some quantile levels, neural statistical functions achieve this with fewer evaluations. 

Additionally, we also provide error visualization in Figure~\ref{fig:aero_sample}, where our model is more accurate and spatially coherent than the MC ($K=10$) error map, especially around complex wing regions where the pressure field changes rapidly. This suggests that neural statistical functions capture structured spatial variation in the conditional quantile more reliably than sparse Monte Carlo sampling.

\subsection{Maximum of Crash Processes}\label{sec:crash}

\paragraph{Setup} 

\begin{wrapfigure}{r}{0.45\textwidth}
    \vspace{-20pt}
    \centering
    \includegraphics[width=0.45\textwidth]{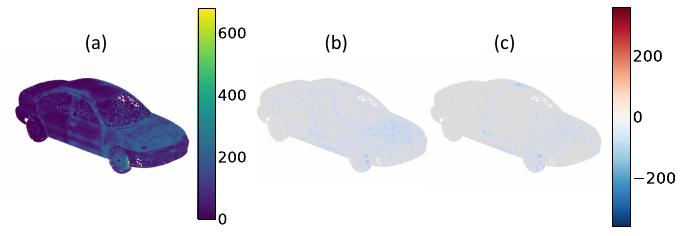}
    \vspace{-15pt}
    \caption{Example maximum stress field over a $100^\circ$ impact angle interval: (a) dense Transolver reference, (b) neural statistical function error, and (c) MC $K=10$ error, both relative to the dense Transolver reference.}
    \label{fig:crash_sample}
    \vspace{-10pt}
\end{wrapfigure}
We evaluate the maximum estimate conditioned on intervals in the Car-Crash
simulated with OpenRadioss \cite{altair_physicsai}. This benchmark contains 280 cases for the industrial-standard National Crash Analysis Center Neon model, where each sample is tested with different impact angles and will present different geometry deformations during crash. Here, the amount of interest $\boldsymbol{h}(\mathbf{x},c)$ is the
element-wise maximum 2D von Mises stress during the crash process and $\mathbf{x}$ denotes the geometry position. The uncertain parameter $c$ is the
impact angle over $[c_{\min},c_{\max}]=[-55^\circ,55^\circ]$, treated as
uniform, and queries are defined over intervals $[c_0,c_1]$ of varying width. 
We adopt 80\% samples (224 cases) for training and the other 20\% samples (56 cases) for testing. The targeted statistic is element-wise maximum stress under different impact angle intervals $[c_0,c_1]$.

As discussed before, we adopt the two approaches to calculate the reference statistics. For dense inference of single-sample predictors, we adopt the 200-sample MC of pre-trained Transolver \cite{wu2024Transolver} to generate a reference, where the same architecture is adopted in training neural statistical functions $\tilde{\boldsymbol{h}}_\theta$. Another reference is calculated from the total 280 numerical simulations for dense evaluation.

\vspace{-5pt}
\paragraph{$\boldsymbol{\psi}$ implementation} As formalized in Eq.~\eqref{eq:max_soft_definition}, we adopt a log-sum-exp formalization \citep{nesterov2005smooth} to approximate the maximum function, which introduces $\beta$ as a smoothing temperature. Specifically, we set $\beta=10$ based on training loss curves to balance training stability and approximation accuracy.

\begin{figure}
    \centering

    \includegraphics[width=\linewidth]{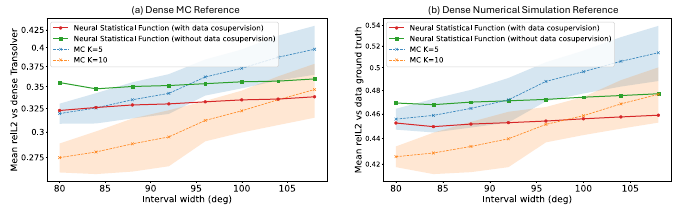}
    \vspace{-10pt}

\caption{Interval-conditioned maximum stress prediction on the Car-Crash. Width-binned relative $\ell_2$ error is reported for neural statistical functions and MC with $K\in\{5,10\}$. 
(a) Errors against the dense Transolver reference. 
(b) Errors against the dense numerical-simulation reference.}
    \vspace{-10pt}
    \label{fig:crash max}
\end{figure}

\vspace{-5pt}
\paragraph{Results}
As shown in Figure~\ref{fig:crash max}(a), neural statistical functions are competitive with or better than the MC $K=10$ baselines while requiring only endpoint evaluations of $\tilde{\boldsymbol{h}}_\theta$. This advantage becomes more evident on wider impact-angle intervals, where the maxima depends on a broader range of crash responses and finite-sample Monte Carlo is not sufficient for this extremely non-linear scenario. 

\vspace{-5pt}
\paragraph{Effect of hybrid supervision} 
As an ablation, we further examine the role of data supervision in model training in Figure~\ref{fig:crash max}(b). Under the ground truth data reference, the neural statistical function with data supervision improves over pure neural predictor supervision, indicating that the additional simulation anchors help align the learned statistic with the real stress envelope observed in the dataset.

\section{Conclusion and Discussion}\label{sec:discussion}

This paper presents neural statistical functions as a new family of neural networks, which can directly infer statistical properties among samples over continuous conditions, thereby significantly reducing computational complexity for stochastic analysis. Specifically, based on the new prefix perspective, we develop a prefix statistics framework to unify diverse statistical functions and further derive an easy-to-optimize learning objective, achieving favorable generality and training efficiency. Neural statistical functions accelerate statistics estimation of complex physics processes by up to 100$\times$, demonstrating strong potential for design and planning that require handling multiple circumstances.

\vspace{-5pt}
\paragraph{Limitations} In this paper, we only experiment on tasks with condition variables that can be represented by a one-dimensional condition variable. However, the prefix statistics framework is also conceptually extensible to multi-dimensional conditions, namely multiple condition variables.

Let $\mathbf{s}=(s^{(1)},\ldots,s^{(d_{\text{cond}})})\in[0,1]^{d_{\text{cond}}}$ denote the marginal quantile
coordinates and define
\begin{equation}
    \tilde{\boldsymbol{h}}(\mathbf{x},\mathbf{s})
    =
    (\nicefrac{1}{\prod_{j=1}^{d_{\text{cond}}} s^{(j)}})
    \int_{0}^{s^{(1)}}\cdots\int_{0}^{s^{({d_{\text{cond}}})}}
    \boldsymbol{\psi}(\mathbf{x},\mathbf{t})\,\dd \mathbf{t}.
\end{equation}
For a box query $[\mathbf{s}_0,\mathbf{s}_1]$, the corresponding interval integral can be recovered by inclusion--exclusion over the $2^{d_{\text{cond}}}$ corners of the box. Therefore, the inference cost increases from two prefix evaluations in the one-dimensional case to $2^{d_{\text{cond}}}$ evaluations in the ${d_{\text{cond}}}$-dimensional case, which is still acceptable for common use cases with $d_{\mathrm{cond}}\leq 4$. It is also notable that, although Monte Carlo sampling is more scalable in high-dimensional settings, it still suffers from sampling noise and slow convergence~\cite{caflisch1998monte,robert2004monte}. Thus, designing sampling-free, low-latency neural statistical functions for extremely high-dimensional condition spaces remains an important open challenge in current research.

\newpage
\bibliographystyle{plainnat}
\bibliography{references}

\newpage
\appendix
\section{Proof of Main Text}\label{appdix:proof}
This section provides proofs for propositions and theorems in the main text.

\subsection{Proof of Proposition~\ref{prop:ode_prefix_equivalence}}
\label{app:proof_ode_prefix_equivalence}

\begin{proof}
We fix $\mathbf{x}$ throughout the proof and omit it from the notation
whenever no confusion arises.

We first show that the prefix-integral representation implies the
first-order identity. Suppose that
\[
    \tilde{\boldsymbol{h}}(\mathbf{x},s)
    =
    \frac{1}{s}\int_0^s \boldsymbol{\psi}(\mathbf{x},t)\,\dd t,
    \qquad s\in(0,1].
\]
Then
\[
    s\,\tilde{\boldsymbol{h}}(\mathbf{x},s)
    =
    \int_0^s \boldsymbol{\psi}(\mathbf{x},t)\,\dd t .
\]
By the fundamental theorem of calculus,
\[
    \frac{\dd}{\dd s}
    \Bigl[
    s\,\tilde{\boldsymbol{h}}(\mathbf{x},s)
    \Bigr]
    =
    \boldsymbol{\psi}(\mathbf{x},s)
    \qquad \text{for a.e. } s\in(0,1].
\]
Using the product rule, we obtain
\[
    s\,\partial_s \tilde{\boldsymbol{h}}(\mathbf{x},s)
    +
    \tilde{\boldsymbol{h}}(\mathbf{x},s)
    =
    \boldsymbol{\psi}(\mathbf{x},s),
\]
which is precisely Eq.~\eqref{eq:running_identity}. Moreover,
\[
    \lim_{s\to 0}
    s\,\tilde{\boldsymbol{h}}(\mathbf{x},s)
    =
    \lim_{s\to 0}
    \int_0^s \boldsymbol{\psi}(\mathbf{x},t)\,\dd t
    =
    0,
\]
where the last equality follows from
$\boldsymbol{\psi}(\mathbf{x},\cdot)\in L^1(0,1]$.

Conversely, suppose that $\tilde{\boldsymbol{h}}$ satisfies
Eq.~\eqref{eq:running_identity} on $(0,1]$ together with the boundary
condition
\[
    \lim_{s\to 0}
    s\,\tilde{\boldsymbol{h}}(\mathbf{x},s)
    =
    0.
\]
By the product rule,
\[
    \frac{\dd}{\dd s}
    \Bigl[
    s\,\tilde{\boldsymbol{h}}(\mathbf{x},s)
    \Bigr]
    =
    s\,\partial_s\tilde{\boldsymbol{h}}(\mathbf{x},s)
    +
    \tilde{\boldsymbol{h}}(\mathbf{x},s).
\]
Using Eq.~\eqref{eq:running_identity}, this gives
\[
    \frac{\dd}{\dd s}
    \Bigl[
    s\,\tilde{\boldsymbol{h}}(\mathbf{x},s)
    \Bigr]
    =
    \boldsymbol{\psi}(\mathbf{x},s).
\]
Integrating both sides from $u$ to $s$, where $0<u<s\le 1$, yields
\[
    s\,\tilde{\boldsymbol{h}}(\mathbf{x},s)
    -
    u\,\tilde{\boldsymbol{h}}(\mathbf{x},u)
    =
    \int_u^s
    \boldsymbol{\psi}(\mathbf{x},t)\,\dd t .
\]
Taking the limit $u\to 0$ and using the boundary condition, we obtain
\[
    s\,\tilde{\boldsymbol{h}}(\mathbf{x},s)
    =
    \int_0^s
    \boldsymbol{\psi}(\mathbf{x},t)\,\dd t .
\]
Dividing both sides by $s$ proves
\[
    \tilde{\boldsymbol{h}}(\mathbf{x},s)
    =
    \frac{1}{s}
    \int_0^s
    \boldsymbol{\psi}(\mathbf{x},t)\,\dd t,
\]
which is Eq.~\eqref{equ:pre_stat}. This completes the proof.
\end{proof}

\subsection{Proof of Theorem~\ref{thm:residual_query_error}}
\label{app:proof_residual_query_error}

\begin{proof}
Fix $\mathbf{x}$ and omit it from the notation when no confusion arises.
By definition, the residual is
\[
    \boldsymbol{\delta}_\theta(\mathbf{x},s)
    =
    s\,\partial_s \tilde{\boldsymbol{h}}_\theta(\mathbf{x},s)
    +
    \tilde{\boldsymbol{h}}_\theta(\mathbf{x},s)
    -
    \boldsymbol{\psi}(\mathbf{x},s).
\]
Using the product rule, we have
\[
    \frac{\dd}{\dd s}
    \Bigl[
    s\,\tilde{\boldsymbol{h}}_\theta(\mathbf{x},s)
    \Bigr]
    =
    s\,\partial_s \tilde{\boldsymbol{h}}_\theta(\mathbf{x},s)
    +
    \tilde{\boldsymbol{h}}_\theta(\mathbf{x},s).
\]
Therefore,
\[
    \boldsymbol{\delta}_\theta(\mathbf{x},s)
    =
    \frac{\dd}{\dd s}
    \Bigl[
    s\,\tilde{\boldsymbol{h}}_\theta(\mathbf{x},s)
    \Bigr]
    -
    \boldsymbol{\psi}(\mathbf{x},s).
\]
Integrating both sides over $[s_0,s_1]$ gives
\[
    \int_{s_0}^{s_1}
    \boldsymbol{\delta}_\theta(\mathbf{x},s)\,\dd s
    =
    \int_{s_0}^{s_1}
    \frac{\dd}{\dd s}
    \Bigl[
    s\,\tilde{\boldsymbol{h}}_\theta(\mathbf{x},s)
    \Bigr]
    \,\dd s
    -
    \int_{s_0}^{s_1}
    \boldsymbol{\psi}(\mathbf{x},s)\,\dd s .
\]
Hence,
\[
    \int_{s_0}^{s_1}
    \boldsymbol{\delta}_\theta(\mathbf{x},s)\,\dd s
    =
    s_1\tilde{\boldsymbol{h}}_\theta(\mathbf{x},s_1)
    -
    s_0\tilde{\boldsymbol{h}}_\theta(\mathbf{x},s_0)
    -
    \int_{s_0}^{s_1}
    \boldsymbol{\psi}(\mathbf{x},s)\,\dd s .
\]
Dividing both sides by $m=s_1-s_0$ yields
\[
    \frac{
    s_1\tilde{\boldsymbol{h}}_\theta(\mathbf{x},s_1)
    -
    s_0\tilde{\boldsymbol{h}}_\theta(\mathbf{x},s_0)
    }{m}
    -
    \frac{1}{m}
    \int_{s_0}^{s_1}
    \boldsymbol{\psi}(\mathbf{x},s)\,\dd s
    =
    \frac{1}{m}
    \int_{s_0}^{s_1}
    \boldsymbol{\delta}_\theta(\mathbf{x},s)\,\dd s .
\]
Taking norms on both sides and applying the triangle inequality gives
\[
    \left\|
    \frac{
    s_1 \tilde{\boldsymbol{h}}_\theta(\mathbf{x},s_1)
    -
    s_0 \tilde{\boldsymbol{h}}_\theta(\mathbf{x},s_0)
    }{m}
    -
    \frac{1}{m}
    \int_{s_0}^{s_1}
    \boldsymbol{\psi}(\mathbf{x},s)\,\dd s
    \right\|
    =
    \left\|
    \frac{1}{m}
    \int_{s_0}^{s_1}
    \boldsymbol{\delta}_\theta(\mathbf{x},s)\,\dd s
    \right\|
\]
and therefore
\[
    \left\|
    \frac{
    s_1 \tilde{\boldsymbol{h}}_\theta(\mathbf{x},s_1)
    -
    s_0 \tilde{\boldsymbol{h}}_\theta(\mathbf{x},s_0)
    }{m}
    -
    \frac{1}{m}
    \int_{s_0}^{s_1}
    \boldsymbol{\psi}(\mathbf{x},s)\,\dd s
    \right\|
    \le
    \frac{1}{m}
    \int_{s_0}^{s_1}
    \left\|
    \boldsymbol{\delta}_\theta(\mathbf{x},s)
    \right\|\,\dd s.
\]
This proves Eq.~\eqref{eq:normalized_residual_error_bound}.
\end{proof}

\subsection{Proof of Proposition \ref{prop:narrow_interval_sensitivity}}
\begin{proof}
Fix $\mathbf{x}$ and omit it from the notation when no confusion arises.
By definition,
\[
    \boldsymbol{\epsilon}_\theta(\mathbf{x},s)
    =
    \tilde{\boldsymbol{h}}_\theta(\mathbf{x},s)
    -
    \tilde{\boldsymbol{h}}(\mathbf{x},s).
\]
Therefore,
\[
\begin{aligned}
\frac{
s_1 \tilde{\boldsymbol{h}}_\theta(\mathbf{x},s_1)
-
s_0 \tilde{\boldsymbol{h}}_\theta(\mathbf{x},s_0)
}{s_1-s_0}
-
\frac{
s_1 \tilde{\boldsymbol{h}}(\mathbf{x},s_1)
-
s_0 \tilde{\boldsymbol{h}}(\mathbf{x},s_0)
}{s_1-s_0}
=
\frac{
s_1 \boldsymbol{\epsilon}_\theta(\mathbf{x},s_1)
-
s_0 \boldsymbol{\epsilon}_\theta(\mathbf{x},s_0)
}{s_1-s_0}.
\end{aligned}
\]
Define
\[
    \boldsymbol{g}_\theta(\mathbf{x},s)
    :=
    s\,\boldsymbol{\epsilon}_\theta(\mathbf{x},s).
\]
Since $\boldsymbol{\epsilon}_\theta(\mathbf{x},\cdot)$ is differentiable,
$\boldsymbol{g}_\theta(\mathbf{x},\cdot)$ is also differentiable on
$(0,1)$. Hence the last term can be written as the difference quotient
\[
    \frac{
    \boldsymbol{g}_\theta(\mathbf{x},s_1)
    -
    \boldsymbol{g}_\theta(\mathbf{x},s_0)
    }{s_1-s_0}.
\]
Taking the limit as $s_0,s_1\to s^\star$ with $s_1>s_0$, we obtain
\[
    \lim_{\substack{s_0,s_1\to s^\star\\ s_1>s_0}}
    \frac{
    \boldsymbol{g}_\theta(\mathbf{x},s_1)
    -
    \boldsymbol{g}_\theta(\mathbf{x},s_0)
    }{s_1-s_0}
    =
    \partial_s \boldsymbol{g}_\theta(\mathbf{x},s^\star).
\]
Using the product rule,
\[
    \partial_s \boldsymbol{g}_\theta(\mathbf{x},s^\star)
    =
    \partial_s
    \Bigl[
    s\,\boldsymbol{\epsilon}_\theta(\mathbf{x},s)
    \Bigr]_{s=s^\star}
    =
    \boldsymbol{\epsilon}_\theta(\mathbf{x},s^\star)
    +
    s^\star
    \partial_s \boldsymbol{\epsilon}_\theta(\mathbf{x},s^\star).
\]
Combining the above identities gives
\[
\lim_{\substack{s_0,s_1\to s^\star,\, s_1>s_0}}
\frac{
s_1 \tilde{\boldsymbol{h}}_\theta(\mathbf{x},s_1)
-
s_0 \tilde{\boldsymbol{h}}_\theta(\mathbf{x},s_0)
}{s_1-s_0}
-
\frac{
s_1 \tilde{\boldsymbol{h}}(\mathbf{x},s_1)
-
s_0 \tilde{\boldsymbol{h}}(\mathbf{x},s_0)
}{s_1-s_0}
=
\boldsymbol{\epsilon}_\theta(\mathbf{x},s^\star)
+
s^\star \partial_s\boldsymbol{\epsilon}_\theta(\mathbf{x},s^\star),
\]
which proves Eq.~\eqref{eq:narrow_interval_limit}.
\end{proof}

\section{More Results}

In this section, we will first present more ablations and model analysis and then provide the quantitative results for the main text figures.

\subsection{Full Statistics for Benchmarks}\label{appdix:full_stat}

As a supplement to Figure \ref{fig:teaser}, we provide the remaining evaluation of all three statistics here. Specifically, the reduction in the number of model inferences is tested under the largest condition interval.

\begin{figure}[htbp]
    \centering
    \includegraphics[width=\linewidth]{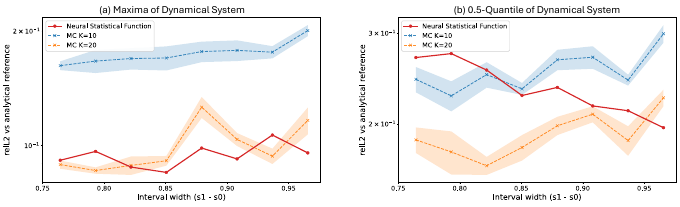}
    \vspace{-10pt}
        \caption{Maxima and $0.5$-Quantile of 2D dynamical system. References are computed analytically from the ground-truth dynamics.}
    \label{fig:aero_appendix}
    \vspace{-10pt}
\end{figure}

\begin{figure}[htbp]
    \centering
    \includegraphics[width=\linewidth]{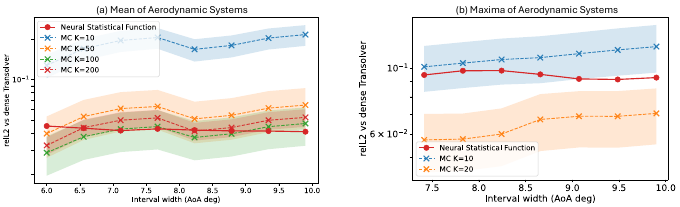}
    \vspace{-10pt}
        \caption{Mean and maxima of aerodynamics. Reference is dense MC of single-condition model.}
    \label{fig:aero_appendix}
    \vspace{-10pt}
\end{figure}

\begin{figure}[htbp]
    \centering
    \includegraphics[width=\linewidth]{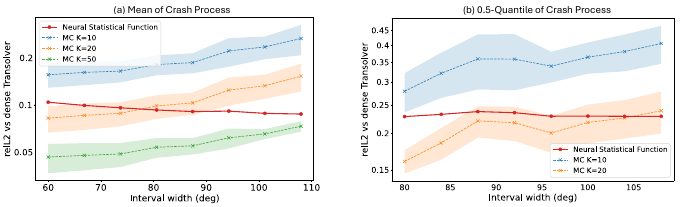}
    \vspace{-10pt}
        \caption{Mean and $0.5$-Quantile of crash. Reference is dense MC of single-condition model.}
    \label{fig:aero_appendix}
    \vspace{-10pt}
\end{figure}

\subsection{Model Analysis}
\label{app:spiral_diagnostics}

\begin{figure*}[t]
    \centering
    \includegraphics[width=\textwidth]{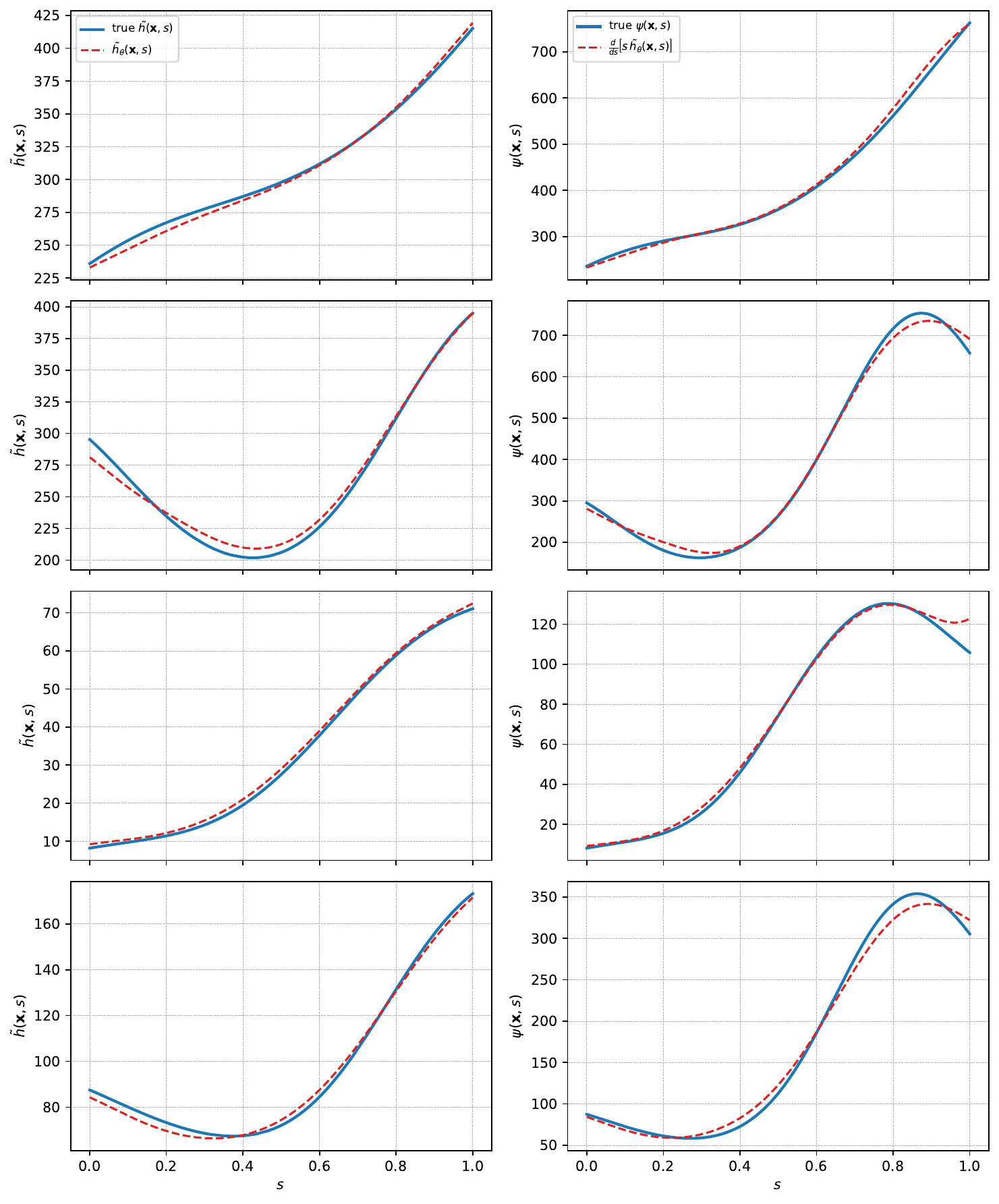}
    \caption{Per-trajectory diagnostics for representative held-out trajectories in the 2D dynamical-system experiment. Each row corresponds to one trajectory. Left: learned prefix statistic $\tilde{\boldsymbol{h}}_\theta(\mathbf{x},s)$ compared with the dense reference $\tilde{\boldsymbol{h}}(\mathbf{x},s)=\frac{1}{s}\int_0^s \boldsymbol{h}(\mathbf{x},t)\,\dd t$. Right: reconstructed local signal $\frac{\dd}{\dd s}[s\,\tilde{\boldsymbol{h}}_\theta(\mathbf{x},s)]$ compared with the true transformed signal $\boldsymbol{\psi}(\mathbf{x},s)=\boldsymbol{h}(\mathbf{x},s)$. Across trajectories, the surrogate tracks both the prefix statistic and the local differential structure closely.}
    \label{fig:spiral_appendix_diagnostics}
\end{figure*}

\paragraph{Per-Trajectory diagnostics for the 2D dynamical system} To further inspect what neural statistical functions learn beyond interval-level error metrics, we visualize per-trajectory diagnostics on several held-out trajectories from the 2D dynamical-system experiment. For each selected trajectory, we compare both the learned prefix statistic and the corresponding reconstructed local signal against their reference values. Specifically, the left column of Figure~\ref{fig:spiral_appendix_diagnostics} shows the learned prefix statistic $\tilde{\boldsymbol{h}}_\theta(\mathbf{x},s)$ against the dense numerical reference
\[
\tilde{\boldsymbol{h}}(\mathbf{x},s)
=
\frac{1}{s}\int_0^s \boldsymbol{h}(\mathbf{x},t)\,\dd t,
\]
while the right column compares the reconstructed local signal
\[
\frac{\dd}{\dd s}\!\left[s\,\tilde{\boldsymbol{h}}_\theta(\mathbf{x},s)\right]
\]
with the true transformed signal $\boldsymbol{\psi}(\mathbf{x},s)=\boldsymbol{h}(\mathbf{x},s)$. The reconstruction is obtained by finite differences in $s$, matching the implementation used during training. These diagnostics complement the interval-query results in Figure~\ref{fig:spiral_mean}. Agreement in the prefix statistic confirms that endpoint evaluations of $\tilde{\boldsymbol{h}}_\theta$ recover interval means accurately, while agreement in the reconstructed local signal verifies that the learned surrogate approximately satisfies the first-order identity in Eq.~\eqref{eq:running_identity}.

\begin{wrapfigure}{r}{0.45\textwidth}
    \vspace{-10pt}
    \centering
    \includegraphics[width=0.45\textwidth]{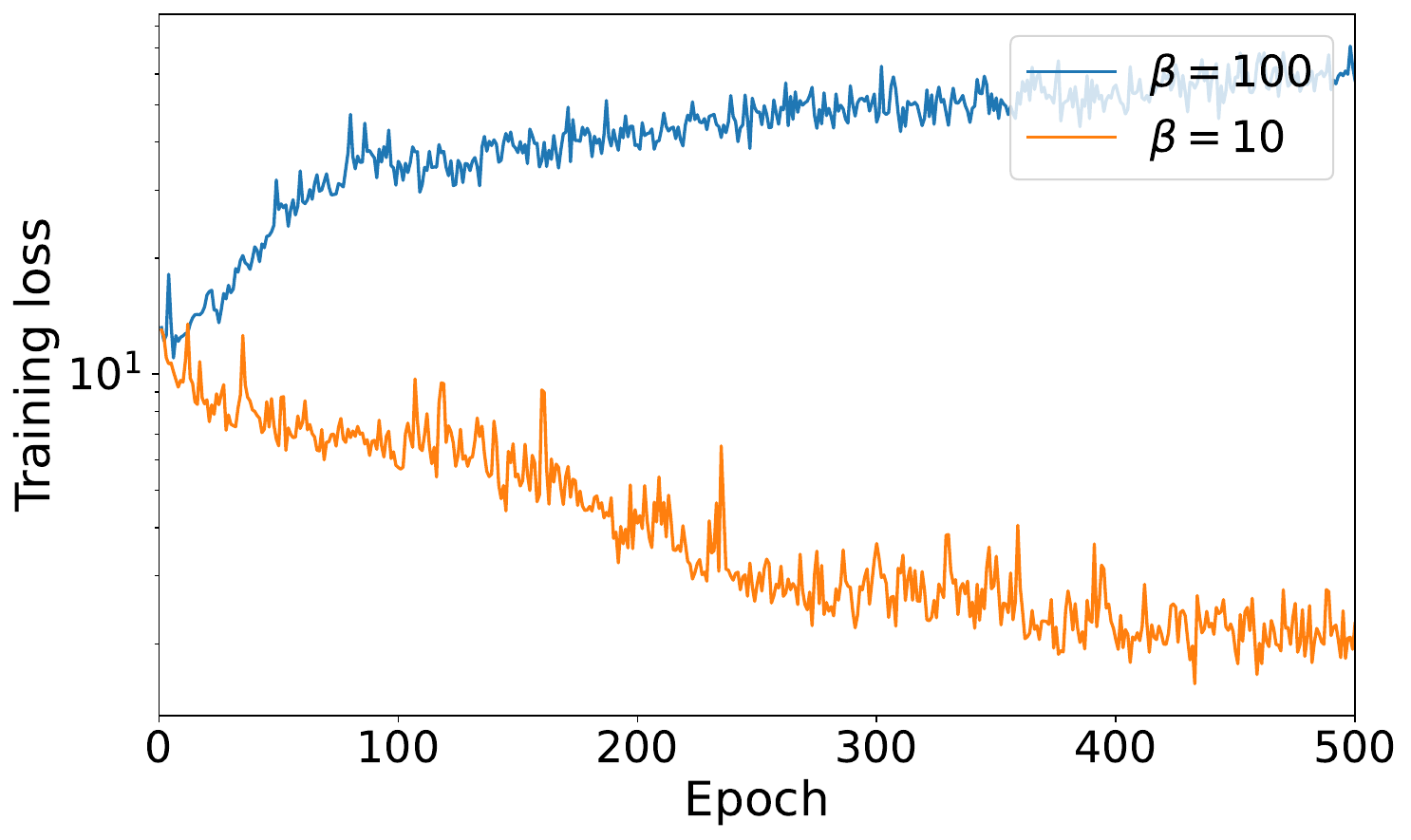}
    \vspace{-15pt}
    \caption{Training loss over the first 500 epochs for $\beta{=}100$ and $\beta{=}10$.}
    \label{fig:beta_ablation_losses}
    \vspace{-10pt}
\end{wrapfigure}
\vspace{-5pt}
\paragraph{Ablations of $\beta$ on Car-Crash maxima} 
The log-sum-exp approximation in Eq.~\eqref{eq:max_soft_definition} introduces a smoothing temperature $\beta$. Increasing $\beta$ reduces the bias between the soft maximum and the hard maximum, but it also amplifies variation in the transformed local signal $\boldsymbol{\psi}(\mathbf{x},s)=\exp(\beta \boldsymbol{h}(\mathbf{x},s))$. In particular, large $\beta$ concentrates the interval statistic on a small subset of high-stress angles, so small errors in the single-sample response or in the finite-difference estimate of $\partial_s\tilde{\boldsymbol{h}}_\theta$ can produce large changes in the identity target. This leads to higher-variance gradients and less stable optimization. In our experiments, large $\beta$ values produced noisier convergence and occasional loss spikes, whereas smaller $\beta$ values were easier to train but yielded an overly smoothed maximum. Based on this trade-off, we set $\beta=10$ for the experiments.

\begin{wrapfigure}{r}{0.45\textwidth}
    \vspace{-10pt}
    \centering
    \includegraphics[width=0.45\textwidth]{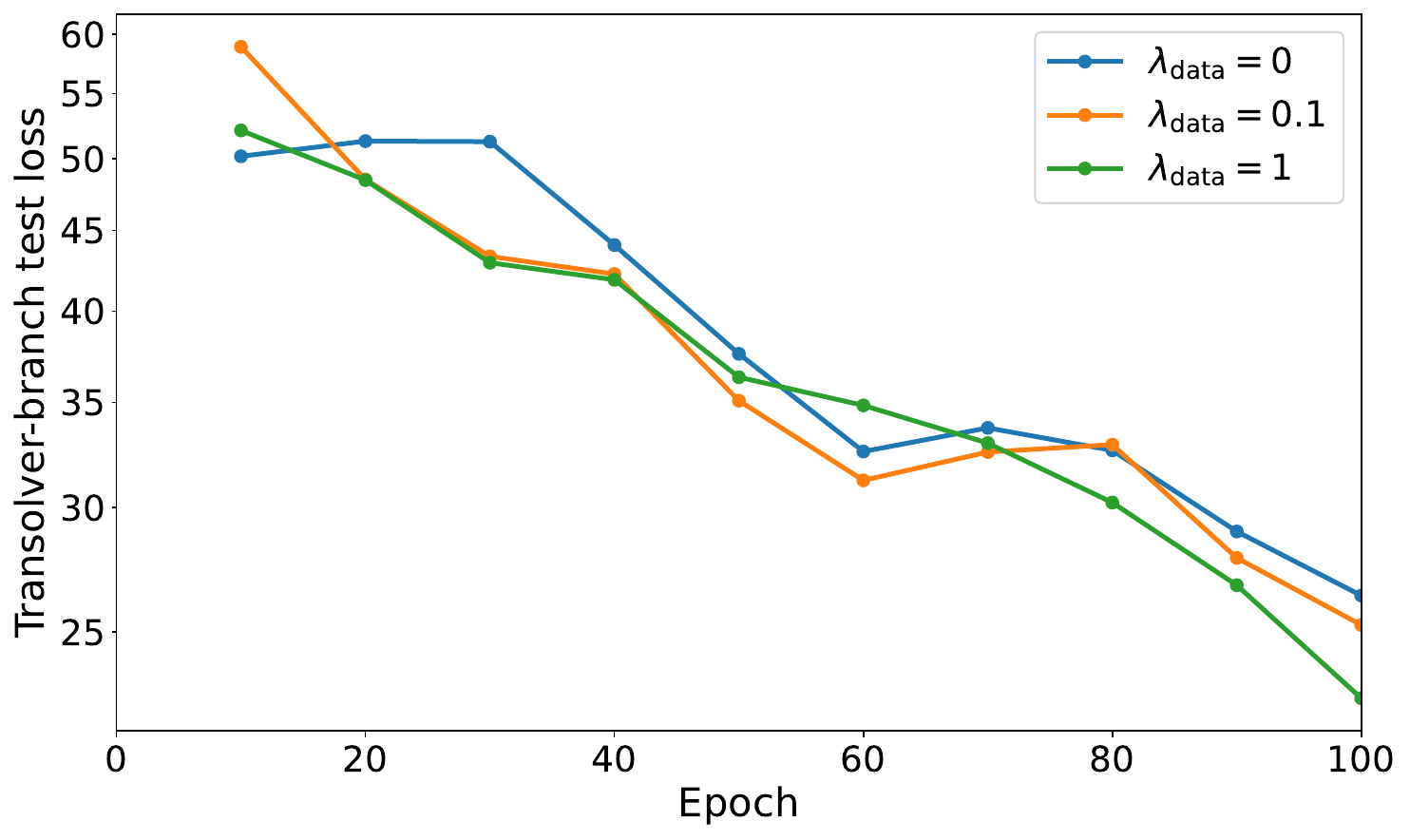}
    \vspace{-15pt}
    \caption{$\mathcal{L}_{\text{neural}}$ on test set over the first 100 epochs for $\lambda_\text{data} = 0$, $\lambda_\text{data}= 0.1$, $\lambda_\text{data} =1$.}
    \label{fig:lamdba}
    \vspace{-10pt}
\end{wrapfigure}
\vspace{-5pt}
\paragraph{Ablations on $\lambda_{\text{data}}$ in hybrid supervision} We use $\lambda_{\text{data}}=0.1$ as the default setting without extensive hyperparameter tuning and observe a clear improvement. To examine the effect of this weight, we conduct an ablation by reporting the Transolver-branch test loss $\mathcal{L}_{\text{neural}}$ when training neural statistical functions for maximum prediction in the crash process, as illustrated in Figure~\ref{fig:lamdba}.

The loss curves show that including the data loss term ($\lambda_{\text{data}}>0$) improves optimization of the neural branch, indicating that the neural and data losses provide complementary supervision. From the loss curve in Figure \ref{fig:lamdba}, we can observe that adjusting $\lambda_{\text{data}}$, \emph{e.g.}, $\lambda_{\text{data}}=1$, may further reduce the test loss and improve the model performance. Since we focus primarily on the methodology in this paper, we leave detailed hyperparameter tuning to future work, where some adaptive hyperparameter tuning techniques \cite{kendall2018multi} can be useful.

\subsection{Quantitative Results}
In Table \ref{tab:integral_count_mc_wfrom0p47}-\ref{tab:crash_max_relL2_vs_width_data_gt}, we present numerical values of plotted main text results for quantitative comparison.

\begin{table}[h]
\centering
\caption{Spiral integral relative L2 ($\mathrm{relL2}$) vs. dense Transolver reference.
Neural statistical function vs.\ MC at $K \in \{10, 50, 100, 200, 400\}$.
MC entries are $\text{mean} \pm 1\,\text{standard error (SE)}$ over 5 seeds.}
\vspace{2pt}
\label{tab:integral_count_mc_wfrom0p47}
\setlength{\tabcolsep}{4pt}
\begin{tabular}{c c c c c c c}
\toprule
$s_1-s_0$ & \textbf{Ours} & MC $K{=}10$ & MC $K{=}50$ & MC $K{=}100$ & MC $K{=}200$ & MC $K{=}400$ \\
\midrule
$0.500$ & $0.028$ & $0.096 \pm 0.002$ & $0.042 \pm 0.001$ & $0.031 \pm 0.001$ & $0.021 \pm 0.001$ & $0.016 \pm 0.001$ \\
$0.560$ & $0.035$ & $0.112 \pm 0.006$ & $0.051 \pm 0.001$ & $0.040 \pm 0.001$ & $0.027 \pm 0.001$ & $0.019 \pm 0.001$ \\
$0.620$ & $0.023$ & $0.109 \pm 0.006$ & $0.051 \pm 0.001$ & $0.038 \pm 0.001$ & $0.026 \pm 0.001$ & $0.020 \pm 0.001$ \\
$0.680$ & $0.021$ & $0.129 \pm 0.005$ & $0.057 \pm 0.003$ & $0.038 \pm 0.002$ & $0.031 \pm 0.002$ & $0.019 \pm 0.001$ \\
$0.740$ & $0.031$ & $0.134 \pm 0.004$ & $0.061 \pm 0.002$ & $0.042 \pm 0.002$ & $0.029 \pm 0.001$ & $0.021 \pm 0.001$ \\
$0.800$ & $0.019$ & $0.141 \pm 0.015$ & $0.065 \pm 0.003$ & $0.048 \pm 0.003$ & $0.033 \pm 0.003$ & $0.022 \pm 0.001$ \\
$0.860$ & $0.019$ & $0.172 \pm 0.015$ & $0.070 \pm 0.002$ & $0.049 \pm 0.002$ & $0.034 \pm 0.001$ & $0.024 \pm 0.001$ \\
$0.920$ & $0.022$ & $0.173 \pm 0.011$ & $0.080 \pm 0.005$ & $0.054 \pm 0.002$ & $0.041 \pm 0.002$ & $0.025 \pm 0.000$ \\
\bottomrule
\end{tabular}
\end{table}
\begin{table}[h]
\centering
\caption{Spiral integral relative L2 ($\mathrm{relL2}$) vs. analytic ground truth.
Neural statistical function vs.\ MC at $K \in \{10, 50, 100, 200, 400\}$.
MC entries are $\text{mean} \pm 1\,\text{standard error}$ over 5 seeds.}
\vspace{2pt}
\label{tab:integral_analytic_wfrom0p47}
\setlength{\tabcolsep}{4pt}
\begin{tabular}{c c c c c c c}
\toprule
$s_1-s_0$ & \textbf{Ours} & MC $K{=}10$ & MC $K{=}50$ & MC $K{=}100$ & MC $K{=}200$ & MC $K{=}400$ \\
\midrule
$0.500$ & $0.028$ & $0.096 \pm 0.002$ & $0.042 \pm 0.001$ & $0.031 \pm 0.001$ & $0.021 \pm 0.001$ & $0.016 \pm 0.001$ \\
$0.560$ & $0.036$ & $0.112 \pm 0.006$ & $0.051 \pm 0.001$ & $0.040 \pm 0.001$ & $0.027 \pm 0.001$ & $0.019 \pm 0.001$ \\
$0.620$ & $0.023$ & $0.109 \pm 0.006$ & $0.051 \pm 0.001$ & $0.038 \pm 0.001$ & $0.026 \pm 0.001$ & $0.020 \pm 0.001$ \\
$0.680$ & $0.021$ & $0.129 \pm 0.005$ & $0.057 \pm 0.003$ & $0.038 \pm 0.002$ & $0.031 \pm 0.002$ & $0.019 \pm 0.001$ \\
$0.740$ & $0.031$ & $0.134 \pm 0.004$ & $0.061 \pm 0.002$ & $0.042 \pm 0.002$ & $0.029 \pm 0.001$ & $0.021 \pm 0.001$ \\
$0.800$ & $0.019$ & $0.141 \pm 0.015$ & $0.065 \pm 0.003$ & $0.048 \pm 0.003$ & $0.033 \pm 0.003$ & $0.022 \pm 0.001$ \\
$0.860$ & $0.019$ & $0.172 \pm 0.015$ & $0.070 \pm 0.002$ & $0.049 \pm 0.002$ & $0.034 \pm 0.001$ & $0.024 \pm 0.001$ \\
$0.920$ & $0.022$ & $0.173 \pm 0.011$ & $0.080 \pm 0.005$ & $0.054 \pm 0.002$ & $0.041 \pm 0.002$ & $0.025 \pm 0.000$ \\
\bottomrule
\end{tabular}
\end{table}

\begin{table}
\centering
\caption{NASA-CRM quantile field $\mathrm{relL2}$ versus dense Transolver reference.
  Neural statistical function vs.\ MC at $K{=}10$ and $K{=}20$ Transolver calls. MC entries are $\text{mean} \pm 1\,\text{SE}$ over 5 seeds.}
  \vspace{2pt}
\label{tab:quantile_line_plot}
\setlength{\tabcolsep}{5pt}
\begin{tabular}{cc ccc}
\toprule
$\alpha$-Quantile & Interval width (deg) & Neural statistical function & MC $K{=}10$ & MC $K{=}20$ \\
\midrule
$0.5$ & $8.000$ & $0.199$ & $0.441 \pm 0.083$ & $0.242 \pm 0.036$ \\
 & $8.475$ & $0.201$ & $0.533 \pm 0.091$ & $0.268 \pm 0.037$ \\
 & $8.950$ & $0.205$ & $0.574 \pm 0.096$ & $0.282 \pm 0.038$ \\
 & $9.425$ & $0.204$ & $0.576 \pm 0.097$ & $0.292 \pm 0.040$ \\
 & $9.900$ & $0.203$ & $0.563 \pm 0.097$ & $0.297 \pm 0.042$ \\
\midrule
$0.7$ & $8.000$ & $0.235$ & $0.328 \pm 0.081$ & $0.154 \pm 0.039$ \\
 & $8.475$ & $0.242$ & $0.346 \pm 0.084$ & $0.164 \pm 0.037$ \\
 & $8.950$ & $0.244$ & $0.361 \pm 0.086$ & $0.172 \pm 0.038$ \\
 & $9.425$ & $0.248$ & $0.380 \pm 0.093$ & $0.178 \pm 0.039$ \\
 & $9.900$ & $0.249$ & $0.399 \pm 0.099$ & $0.185 \pm 0.041$ \\
\midrule
$0.9$ & $8.000$ & $0.211$ & $0.213 \pm 0.064$ & $0.096 \pm 0.014$ \\
 & $8.475$ & $0.180$ & $0.199 \pm 0.058$ & $0.090 \pm 0.013$ \\
 & $8.950$ & $0.155$ & $0.201 \pm 0.057$ & $0.091 \pm 0.012$ \\
 & $9.425$ & $0.145$ & $0.214 \pm 0.060$ & $0.096 \pm 0.013$ \\
 & $9.900$ & $0.144$ & $0.231 \pm 0.066$ & $0.102 \pm 0.014$ \\
\bottomrule
\end{tabular}
\end{table}

\begin{table}
\centering
\caption{Crash max $\mathrm{relL2}$ versus interval width,
  evaluated against dense Transolver reference.
  Neural statistical function with and without data cosupervision are compared with MC at $K{=}5$ and $K{=}10$ solver calls. MC entries are $\text{mean} \pm 1\,\text{SE}$.}
  \vspace{2pt}
\label{tab:crash_max_relL2_vs_width_transolver}
\setlength{\tabcolsep}{3pt}
\begin{tabular}{c cc cc}
\toprule
width (deg) & \textbf{Ours} w/ data supervision & \textbf{Ours} w/o data supervision & MC $K{=}5$ & MC $K{=}10$ \\
\midrule
$80.0$  & $0.323$ & $0.352$ & $0.319 \pm 0.011$ & $0.275 \pm 0.014$ \\
$84.0$  & $0.326$ & $0.337$ & $0.325 \pm 0.017$ & $0.280 \pm 0.020$ \\
$88.0$  & $0.329$ & $0.340$ & $0.335 \pm 0.021$ & $0.288 \pm 0.026$ \\
$92.0$  & $0.330$ & $0.341$ & $0.342 \pm 0.023$ & $0.295 \pm 0.028$ \\
$96.0$  & $0.332$ & $0.343$ & $0.362 \pm 0.022$ & $0.312 \pm 0.021$ \\
$100.0$ & $0.335$ & $0.346$ & $0.373 \pm 0.025$ & $0.322 \pm 0.023$ \\
$104.0$ & $0.335$ & $0.346$ & $0.387 \pm 0.030$ & $0.335 \pm 0.028$ \\
$108.0$ & $0.338$ & $0.349$ & $0.398 \pm 0.033$ & $0.347 \pm 0.032$ \\
\bottomrule
\end{tabular}
\end{table}

\begin{table}
\centering
\caption{Crash max $\mathrm{relL2}$ versus interval width,
  evaluated against ground truth data.
  Neural statistical functions with and without data cosupervision are compared with MC at $K{=}5$ and $K{=}10$ solver calls. MC entries are $\text{mean} \pm 1\,\text{SE}$.}
  \vspace{2pt}
\label{tab:crash_max_relL2_vs_width_data_gt}
  \setlength{\tabcolsep}{2.8pt}
\begin{tabular}{c cc cc}
\toprule
width (deg) & \textbf{Ours} w/ data supervision & \textbf{Ours} w/o data supervision & MC $K{=}5$ & MC $K{=}10$ \\
\midrule
$80.0$  & $0.452$ & $0.458$ & $0.456 \pm 0.009$ & $0.426 \pm 0.008$ \\
$84.0$  & $0.449$ & $0.459$ & $0.459 \pm 0.014$ & $0.428 \pm 0.016$ \\
$88.0$  & $0.451$ & $0.461$ & $0.464 \pm 0.016$ & $0.434 \pm 0.020$ \\
$92.0$  & $0.453$ & $0.462$ & $0.472 \pm 0.019$ & $0.440 \pm 0.022$ \\
$96.0$  & $0.454$ & $0.463$ & $0.488 \pm 0.017$ & $0.451 \pm 0.014$ \\
$100.0$ & $0.456$ & $0.465$ & $0.496 \pm 0.019$ & $0.458 \pm 0.016$ \\
$104.0$ & $0.457$ & $0.466$ & $0.506 \pm 0.023$ & $0.468 \pm 0.021$ \\
$108.0$ & $0.459$ & $0.468$ & $0.514 \pm 0.026$ & $0.476 \pm 0.024$ \\
\bottomrule
\end{tabular}
\end{table}

\section{Implementation Details}\label{appdix:imple}

\paragraph{Evaluation protocol}We evaluate all methods on held-out test cases using width-binned relative $\ell_2$ error. For the 2D dynamics modeling (Section \ref{sec:dyn}), we use eight normalized interval widths in $[0.50,0.92]$ and compare against MC baselines with $K\in\{10,50,100,200,400\}$ using both analytic and dense references. For NASA-CRM (Section \ref{sec:aero_sys}), we evaluate quantile levels $\alpha\in\{0.5,0.7,0.9\}$ over AoA interval widths from $8.0^\circ$ to $9.9^\circ$, with a dense-grid Transolver reference computed from 200 uniformly spaced AoA evaluations and MC baselines at $K\in\{10,20\}$. Neural statistical function quantiles are recovered by five bisection steps per node with smoothing temperature $\varepsilon=10^{-2}$; to improve robustness to local non-monotonicities in the learned CDF, bisection uses a running monotone envelope of the predicted CDF values for bracket updates. For the crash (Section \ref{sec:crash}), we evaluate maximum-stress prediction over impact-angle widths from $80^\circ$ to $108^\circ$, using both dense Transolver and data-derived simulation references, and compare with MC baselines at $K\in\{5,10\}$. All MC results are averaged over five random seeds and reported as mean $\pm$ one standard error.

\vspace{-5pt}
\paragraph{Compute resources} All experiments are run on a single multi-GPU workstation with NVIDIA V100 32GB GPUs. Training uses standard PyTorch DDP \cite{Paszke2019PyTorchAI} for faster wall-clock time, while evaluation is performed on a single GPU. Dense Transolver sweeps to generate the reference data of NASA-CRM and Car-Crash dominate the evaluation cost. No multi-node compute cluster is required.

\section{Broader Impacts}\label{appdix:impact}
This paper presents neural statistical functions as a new family of neural networks, whose design principle can be inspiring for future research. Besides, the proposed model can significantly accelerate statistical estimation over a range of varying conditions. One potential application is industrial design, where conventional neural surrogates typically provide predictions under a specific condition, while neural statistical functions allow engineers to directly estimate distributional statistics under diverse operating scenarios. This capability may reduce design-loop latency, improve uncertainty-aware evaluation, and lower the computational cost of large-scale simulation-based analysis. Beyond industrial design, the proposed framework may also benefit decision-making scenarios that require fast statistical estimates, such as motion planning with total energy estimation.

This paper mainly focuses on the scientific problem and we are fully committed to ensuring ethical considerations are taken into account. Thus, we believe there are no potential ethical risks.

\end{document}